\documentclass[11pt,anonymous, a4paper]{article}
\usepackage[hyperref]{emnlp2018}
\usepackage{times}
\usepackage{latexsym}
\usepackage{url}
\usepackage{xcolor}
\usepackage{upgreek}
\usepackage[shortlabels]{enumitem}
\usepackage{tabu}
\setlist[enumerate]{nosep}
\usepackage{floatrow}
\usepackage{multirow}
\usepackage{makecell}
\usepackage{amsmath}

\usepackage{booktabs}

\DeclareFloatFont{tiny}{\tiny}
\usepackage{silence}
\usepackage{tabularx}
\usepackage{pifont}
\usepackage{subcaption}

\newlength\Origarrayrulewidth

\usepackage{todonotes}

\usepackage{amsmath} 
\usepackage{wrapfig}

\renewcommand{\footnotesize}{\fontsize{8pt}{6pt}\selectfont}

\usepackage{graphicx}
\graphicspath{ {./images/} }
\usepackage[ruled]{algorithm2e}
\usepackage{amsthm}
\usepackage{amssymb}
\newtheorem{theorem}{Theorem}[section]
\newtheorem{lemma}{Lemma}[section]
\newtheorem{corollary}{Corollary}[section]
\newtheorem{proposition}{Proposition}[section]
\newtheorem{definition}{Definition}[section]
\makeatletter
\newcommand{\vo}{\vec{o}\@ifnextchar{^}{\,}{}}
\makeatother

\title{ALLWAS: Active Learning on Language models in WASserstein space}
\author{Anson Bastos \\
  {\tt ansonbastos@gmail.com} \\
  IIT, Hyderabad \\
  India \\\And
  Manohar Kaul \\
  {\tt mkaul@iith.ac.in} \\
  IIT, Hyderabad \\
  India}
  
\begin{document}
\setlength{\abovedisplayskip}{3pt}
\setlength{\belowdisplayskip}{3pt}

\maketitle

\begin{abstract}
Active learning has emerged as a standard paradigm in areas with scarcity of labeled training data, such as in the medical domain. 
Language models have emerged as the prevalent choice of several natural language tasks due to the performance boost offered by these models. However, in several domains, such as medicine, the scarcity of labeled training data is a common issue. Also, these models may not work well in cases where class imbalance is prevalent. Active learning may prove helpful in these cases to boost the performance with a limited label budget. 
To this end, we propose a novel method using sampling techniques based on submodular optimization and optimal transport for active learning in language models, dubbed ALLWAS. 
We construct a sampling strategy based on submodular optimization of the designed objective in the gradient domain. Furthermore, to enable learning from few samples, we propose a novel strategy for sampling from the Wasserstein barycenters. Our empirical evaluations on standard benchmark datasets for text classification show that our methods perform significantly better ($>20\%$ relative increase in some cases) than existing approaches for active learning on language models.


\end{abstract}


\section{Introduction} \label{sec:introduction}
Active learning is a technique for improving model performance over a fixed annotation budget \cite{10.5555/1622737.1622744}. Generally, the data is obtained for labeling iteratively after alternating training phases until the desired performance is achieved. This contrasts with passive learning, where one assumes access to labels for the entire pool of data. There are three scenarios for active learning \cite{settles2009active}: (1) \emph{pool-based}, where a set of unlabeled data points are available (2) \emph{stream-based}, in which the data points are received in an online fashion, and (3) \emph{membership query synthesis}, where the data points are generated for labeling. In this work, our focus will be on the pool-based setting for active learning.

Active learning has benefited a wide gamut of applications such as text classification \cite{10.1162/153244302760185243, 10.1145/1135777.1135870}, named entity recognition \cite{Tomanek-etal, shen-etal-2004-multi}, and machine translation \cite{haffari-sarkar-2009-active}, to name a few. Transformer-based language models \cite{devlin-etal-2019-bert} have shown improved performance on NLP tasks. These models with a large number of parameters require comparable amounts of data to produce good results \cite{Margatina/BALM/abs-2104-08320} and thus pose a challenge in the active learning setting.
There has been a recent surge in the study of language models in the active learning setup~\cite{ein-dor-etal-2020-active,Margatina/BALM/abs-2104-08320}.
However, many of these approaches are based on uncertainty sampling, which may not work well for uncalibrated deep models \cite{Guo/Calibration/DBLP:journals/corr/GuoPSW17}. Other approaches look at the embedding space \cite{sener2018active} or the gradient space \cite{Huang/EGL/DBLP:journals/corr/HuangCRLSC16}. However, these methods generally assume a Euclidean metric between the data points in the respective spaces, which fails to judiciously capture the complex interactions. In this paper, we hypothesize that finding a core set of points using the Wasserstein metric would result in better performance than simply selecting a set of points that minimizes or maximizes a measure.

Such large language models require considerable amounts of representative data, which makes it infeasible for fine-tuning in the active learning scenario that work with a limited annotation budget. This drawback is effectively alleviated by \emph{data augmentation} in the image domain~\cite{Ratner/10.5555/3294996.3295083}. However, data augmentation is not a straightforward exercise for textual data. There have been numerous attempts to augment data by generating samples to label in the \emph{token space} \cite{Liu/9240734, quteineh-etal-2020-textual} and the \emph{feature space}~\cite{Kumar/DBLP:journals/corr/abs-1910-04176, feng2021survey}. Generating tokens could render the labels erroneous because of the nature of the hard assignment. To this end, we propose an over-sampling strategy based on \emph{Wasserstein barycenters}~\cite{pmlr-v32-cuturi14} in the embedding space. 
Our rationale here is that augmenting data by such a sampling technique benefits active learning because it operates well in both the \emph{low data regime} as well as the \emph{class imbalance} scenarios.

To the best of our knowledge we are the first to propose such an augmentation method for active learning with language models.
Our key contributions are:
\begin{enumerate}
    \item We propose a novel sampling strategy based on the Wasserstein distance in the gradient space. We prove its submodularity and propose a $1-\frac{1}{e}$ optimal greedy algorithm. 
    \item We design an over-sampling technique based on the Wasserstein barycenter of the embeddings of the data points for better performance in the cases with few labeled samples.
    \item We demonstrate the effectiveness of our method by running extensive experiments on real world scenarios of few labeled samples and class imbalance. We also conduct experiments on the multi-class settings which have not been considered in previous works. 
\end{enumerate}




\section{Related Work} \label{sec:related}
Prior works on active learning have focused on uncertainty based sampling such as entropy~\cite{Lewis/DBLP:journals/corr/LewisG94}, least model confidence \cite{Settles/NIPS2007_a1519de5} and diversity based methods~\cite{Settles/NIPS2007_a1519de5, Xu/10.1007/978-3-540-71496-5_24, Wei/pmlr-v37-wei15}. \cite{Settles/NIPS2007_a1519de5,Hsu/AAAI159636} have tried to use a combination of the diversity and uncertainty based approaches. Active learning has been effectively used in previous works for CNN based models \cite{sener2018active,pmlr-v48-gal16,Gissin/abs-1907-06347}. 
\emph{Coresets} have been used for importance sampling \cite{Cohen/10.5555/3039686.3039801}, $k$-means and medians clustering \cite{Peled/10.1145/1007352.1007400} and for Gaussian mixture models (GMMs)~\cite{Lucic/JMLR:v18:15-506}. Work in \cite{pmlr-v119-mirzasoleiman20a} used coresets in the gradient domain for subsampling data points for accelerated training. 
\cite{Wei/pmlr-v37-wei15} combines the uncertainty sampling methods with a submodular optimization method for subset selection. \cite{ramalingam2021balancing} uses a combination of submodular functions for balancing constraints of class labels and decision boundaries using matroids. A study of the theoretical performance of batch mode active learning with submodularity is given in \cite{Chen/pmlr-v28-chen13b}. Submodular functions have also been used in NLP for text summarization \cite{lin-bilmes-2011-class}, machine translation \cite{Kirchhoff_submodularityforSMT} and goal oriented chatbots \cite{Dimovski/submod_chatbot/10.5555/3304222.3304329}. In contrast to previous works we propose a novel submodular function for query sampling that operates in the gradient space. We argue that this would help select samples that are most representative of the gradients.

Data Augmentation techniques using Wasserstein barycenters and optimal transport have been adopted in the literature \cite{Zhu/10.1007/978-3-030-58607-2_37, bespalov2021data, Nadeem/10.1007/978-3-030-59722-1_35, Yan_Tan_Xu_Cao_Ng_Min_Wu_2019} for the image domain. In contrast, NLP researchers have primarily focused on generating data in the token space \cite{Liu/9240734, quteineh-etal-2020-textual} for data augmentation \cite{wang-yang-2015-thats,kobayashi-2018-contextual}, paraphrase generation \cite{kumar-etal-2019-submodular} etc. There exist methods that use \emph{mixups} in the feature space \cite{Kumar/DBLP:journals/corr/abs-1910-04176, feng2021survey} for data augmentation. However, to the best of our knowledge this is the first work to explore data augmentation using Wasserstein barycenters for active learning using language models. We argue that our method is advantageous in the low sample and imbalanced class settings.




\section{Problem Statement and Approach} \label{sec:problem_approach}
\subsection{Problem Formulation}\label{sec:problem}
Typical pool based active learning methods have the following components: a pool of unlabelled data $U_{pool}$, a model $M$ on which to train the data for the downstream task, and an annotation budget $b$, which is a limit on the amount of labeled data that can be obtained. The last component is an acquisition or query function $q(.)$ that would be used for querying over $U_{pool}$ to obtain the data to be labeled. This is an iterative process in which, at every iteration, the query function $q(.)$ acquires a query set of size $k(<b)$. Finally, the model $M$ is trained over the samples provided by the query function and is evaluated on a validation set $D_{val}$. The aim is to maximize the performance on $D_{val}$ with a minimum labeled sample set \cite{Siddhant/abs-1808-05697}. The process is repeated until the annotation budget $b$ is exceeded or the desired performance on the validation set is achieved. 
\subsection{Approach} \label{sec:approach}
\begin{figure}[htbp]
  \centering
\setkeys{Gin}{width=1.\linewidth}
  \includegraphics{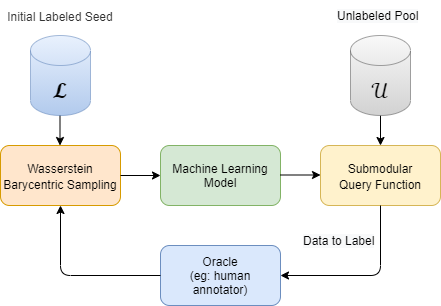}
 \caption{ALLWAS Process Flow. It uses Wasserstein Barycenters for data augmentation. The submodular query function is used for unlabeled data acquisition. These two steps aims to enhance performance of underlying ML model (BERT in our case).}
  \label{fig:AL_process}
\end{figure}

Figure \ref{fig:AL_process} outlines the flow of our proposed approach. First, the initial labeled seed of data is given to the barycentric sampling module for upsampling. Next, a model, in this case, a language model, is trained on this initial seed. The query function then uses the model to sample data points labeled by an oracle and a human annotator. The newly labeled data points are fed to the upsampling module, and the process repeats until the labeling budget is reached or the desired performance is achieved. The details of the individual components, along with some preliminaries, are explained in the following sections.

\subsubsection{Optimal Transport and Wasserstein Barycenters}\label{OT_wass}
Let $\Omega$ be any space, $D$ be a distance metric in $\Omega$, and $P(\Omega)$ be the set of probability measures in that space. Let $x,y \in \Omega$ be the dirac masses with probability measures $\mu$ and $\nu$ respectively. The Optimal transport \cite{monge1781memoire} problem is to minimize the cost in transporting $x$ to $y$ . The Wasserstein distance defines the optimal transport plan to move an amount of matter from one location to another.
\begin{definition}[]
Let $p \in [1,\infty)$ and D: $\Omega \times \Omega$ $\xrightarrow{} [0,\infty)$ be the cost of transporting the measure $\mu$ to $\nu$, then the $p^{th}$ Wasserstein distance \cite{villani2009a} between the measures is given by
\begin{equation} \label{wasserstein}
    W_p(\mu,\nu) = \inf_{\gamma \in \Pi(\mu,\nu)}\left(\int_{\Omega \times \Omega} D(x,y)^p \partial{\gamma}\right)^{\frac{1}{p}}
\end{equation}
where $\Pi$ is the set of all the possible transport plans with the marginals $\mu$ and $\nu$.
\end{definition}

\begin{definition}[Wasserstein Barycenter, \cite{Agueh/doi:10.1137/100805741}]\label{wasserstein_barycenter}
A Wasserstein barycenter of $n$ measures $\{v_1, v_2,..., v_n\}$ in $\mathbb{P} \subset P(\Omega)$ is a measure that minimizes the weighted sum of the $p^{th}$ Wasserstein distance over $\mathbb{P}$ i.e. it is a minimiser of $f$ defined as below
\begin{equation} \label{wasserstein_barycenter_eq}
    f(\mu) = \sum_{i=1}^{N} \lambda_{i} W_{p}^{p} (\mu, \nu_{i})
\end{equation}
\end{definition}
\noindent Here we consider a convex combination of $W_{p}^{p}$ i.e. $\lambda_i \leq 1$ and $\sum_{i} \lambda_i = 1$. If $D$ is the $L_2$ distance and $p=2$ that is when $P(\Omega,D)$ is the euclidean distance metric, minimizing $f$ results in the k means solution \cite{kaufmanl1987clustering}.

\noindent \textbf{Sampling using Wasserstein Barycenters:} The definition of the Wasserstein barycenter in \ref{wasserstein_barycenter} allows us to sample from a set of data points as illustrated in this section.
\begin{figure}[htbp]
  \centering
\setkeys{Gin}{width=1.\linewidth}
  \includegraphics{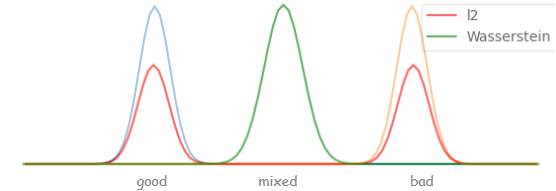}
 \caption{Example of a 1D Distribution (in orange and blue) showing the Wasserstein and $\ell_2$ barycenters}
  \label{fig:wass_vs_l2_bary}
\end{figure}
The intuition for using Wasserstein barycenters instead of euclidean barycenters is outlined in Figure \ref{fig:wass_vs_l2_bary}. The figure shows a distribution of 1-dimensional word vectors, with the words "good" and "bad" at the extremes and the word "mixed" in between them. We see from the distribution that the data contains the words "good" and "bad". Sampling using the Wasserstein barycenter with equal weights gives us the word "mixed." In contrast, the $\ell_2$ barycenter would sample either of the words "good" or "bad" with equal probability. This implies that in a sentiment classification task,  a pair of sentences "The movie was good" and "The movie was bad" would enable us to sample a neutral sentence "The movie was mixed" using the Wasserstein barycenter. We argue that in contrast to sampling from the $\ell_2$ barycenter (see subsection \ref{l2_vs_wass_ablation}) or no over-sampling (augmentation),  sampling technique incorporating Wasserstein barycenters would result in superior performance for the few sample and imbalanced case. 

Consider an machine learning (language) model such as Bert \cite{devlin-etal-2019-bert} with the output $d$ dimensional contextual embeddings as $e_1 , e_2, \cdots e_n$ from a layer $l$ for the input tokens $w_1, w_2, \cdots w_n$ respectively. Now let us consider $s$ sentences each with number of tokens given by $n_1, n_2, \cdots n_s$ respectively. The contextual embeddings of a sentence $s_i$ would be represented as $E_i = [e_1, e_2, \cdots, e_{n_i}]$, where $E_i \in R^{d \times n_i}$. The Wasserstein barycenter of these samples would then be given by,
\begin{equation} \label{wass_bary_sampling}
    E_{c} = \underset{E_{c} \in R^{d \times n_i}}{argmin} \sum_{i=1}^{N} \lambda_{i} W_{p}^{p} (E_{c}, E_i)
\end{equation}
Thus, we obtain the barycenter in the embedding space. We also modify the labels by taking a weighted average as follows:
\begin{equation} \label{labels_sampling}
    L_{c} = \sum_{i=1}^{N} \lambda_{i} L_i
\end{equation}
Where $L_i$ is the true class probabilities of the sentence $s_i$.
Varying the values of the $\lambda$s could enable picking multiple data points, enabling over-sampling from the pool of labeled data.

\subsubsection{Submodular Acquisition function}
Let $V$ be the set of all points in the space $\Omega$ under consideration. Let $A$, $B$ be two subsets of V such that $A \subseteq B$. Let F be a set function (acting on a set $S$) $\Omega^{\vert S \vert} \xrightarrow[]{} R$, then F is said to be submodular if, on adding an element $e \in V \setminus B$ to $A$ and $B$, it satisfies the below condition
\[F\{A \cup e\} - F\{A\} \geq F\{B \cup e\} - F\{B\}\]
Previous works have used gradient spaces for subset selection in active learning \cite{Huang/EGL/DBLP:journals/corr/HuangCRLSC16} and to speed up training \cite{pmlr-v119-mirzasoleiman20a}. Selecting a subset of points with gradients that are representative of the gradients of the entire set of points would intuitively result in steering the model parameters in the right direction of the optimum value. This motivates our approach of using the gradient space to perform the acquisition of the data points. One issue that remains is that unlike in \cite{pmlr-v119-mirzasoleiman20a} we do not have the true labels beforehand. \cite{Huang/EGL/DBLP:journals/corr/HuangCRLSC16} proposes to use the expected gradient length with the expectation over the predicted logits. However, the predicted probabilities do not always correlate with the model confidence \cite{Guo/Calibration/DBLP:journals/corr/GuoPSW17} and calibration of the model may be required. We differ in our approach where we use the Wasserstein distances between the points in the gradient space to find the most representative sample set. Specifically, we select the subset that minimizes the below function.
\begin{equation}\label{eq_ls}
    L\{S\} = \sum_{i \in V} \underset{j \in S}{min} (W_p^p(i,j))
\end{equation}
where $W_p^p(i,j)$ is the $p^{th}$ Wasserstein distance between the $i^{th}$ and the $j^{th}$ sample in the gradient space.
Minimizing L is equivalent to finding the k medoids \cite{kaufmanl1987clustering} and in general, finding an exact solution is an NP-Hard problem. However, optimizing a submodular function enables us to obtain a $1-\frac{1}{e}$ optimal \cite{nemhauser1978a} solution in a greedy manner.
We define a submodular function using $L$ as below:
\begin{equation}\label{submodular_fn_eq}
    F\{S\} = L\{s_0\} - L\{s_0 \cup S\}
\end{equation}
Here $s_0$ is an auxillary set element and $L\{s_0\}$ can be considered a constant. We prove the submodularity of equation \ref{submodular_fn_eq} below.


\begin{lemma}\label{L_dec}
The function $L\{S\}$ is monotone decreasing.
\end{lemma}
\begin{proof}
From the definition of $L$ we have,
\[L = \sum_{i \in V} \underset{j \in S}{min} (W_p^p(i,j))\]
where $W_p^p(i,j)$ is the $p^{th}$ Wasserstein distance in the gradient space.
On adding an element $e \in V \setminus S$ to $S$, we get the new set $S^{'} = S \cup e$. The metric for the new set then becomes $L^{'} = \sum_{i \in V} \underset{j \in S^{'}}{min} (W_p^p(i,j))$.
Now let's assume that $L^{'} > L$. This means that for some point $i \in V$, the newly added point $e$ was selected and the distance $W_p^p(i,e)$ is greater than the previous minimum $W_p^p(i,j)$, which is a contradiction.
Thus, we have that $L^{'} \leq L$. 
\end{proof}

\begin{corollary}\label{F_inc}
The function $F = L\{s_0\} - L\{S \cup s_0\}$ is monotone increasing.
\end{corollary}
\begin{proof}
If we fix $L\{s_0\}$ to a constant and since $L\{S\}$ is monotone decreasing from lemma \ref{L_dec} , we have $F$ is monotone increasing.
\end{proof}

\begin{proposition}\label{prop_rate}
The rate of increase of $F$ at $A \in V$ is greater than or equal to that at $B (\in V) \supseteq A$.
\end{proposition}
To understand proposition \ref{prop_rate} we note that adding $e \in V \setminus B$ to $A$ causes an increase in $F$ or maintains the value as it is monotone increasing from corollary \ref{F_inc} and since $A \subseteq B$ adding $e$ to $B$ will only cause the same or lesser increase in $F\{B\}$ by definition of $F$.

\begin{theorem}
$F$ is a submodular function.
\end{theorem}
\begin{proof}
Assume set of points in the gradient space, $B \in V$ and a set $A \subseteq B$. We assume a continuous space of the elements such that adding a fraction of it would cause a fractional change in the output. Note that an interpolation does not change the function definition for the discrete case where the direct mass is concentrated. Consider adding an element(s) $e \in V \setminus B$ to the sets $A$ and $B$. 
By the gradient theorem for path integral we get,

\small 
\begin{align*}
F(A \cup e) - F(A) &= \int_{0}^{1} \frac{\partial F(A + \alpha(A \cup e - A))}{\partial \alpha} d \alpha \\
  &= \int_{0}^{1} \frac{\partial F(A + \alpha(A \cup e - A))}{\partial x} \frac{\partial x}{\partial \alpha} d \alpha \\
  & \ \ \ \ \ \ \text{(using chain rule)} \\
  &= (A \cup e - A) \int_{0}^{1} \frac{\partial F(A + \alpha(A \cup e - A))}{\partial x}  d \alpha \\
  & \text{(as } \frac{\partial x}{\partial \alpha} = A \cup e - A = \text{constant } K \text{)} \\
  &= K \int_{0}^{1} \frac{\partial F(A + \alpha(A \cup e - A))}{\partial x}  d \alpha
\end{align*}
\normalsize
Similarly,
\small
\[F(B \cup e) - F(B) = K \int_{0}^{1} \frac{\partial F(B + \alpha(B \cup e - B))}{\partial x} d \alpha\]
\normalsize

From proposition \ref{prop_rate} we have $\frac{\partial F(A)}{\partial x}_{|x=A} \geq \frac{\partial F(B)}{\partial x}_{|x=B}$ for addition of the same element $e$ and this will be true in the entire interval $\alpha \in [0,1]$. Thus we get,
\[F(A \cup e) - F(A) \geq F(B \cup e) - F(B)\]
Thus we can conclude that $F$ is submodular.
\end{proof}

Since we have a submodular function in $F$, we could use a greedy algorithm to find a set that is $(1-\frac{1}{e})$ of the optimal set $S$ that maximizes $F$ (minimizes $L$). It further runs in polynomial time. The greedy algorithm begins with an empty set $S=\phi$ and at each iteration keeps adding an element $e \in V \setminus S$ that maximizes $F(e \vert S_{i-1}) = F(e \cup S_{i-1}) - F(S_{i-1})$ i.e. $S_i = S_{i-1} \cup \ \underset{e \in V}{argmax} F(e \vert S_{i-1})$. The iterations continue till a specified labeling budget is attained. In practice, computing the gradients with respect to the entire set of weights could be computationally expensive in language models that could have millions of parameters. Fortunately for deep networks most of the variation in gradients with respect to the loss is captured by the last layer \cite{katharopoulos2019samples}. Also, \citet{pmlr-v119-mirzasoleiman20a} efficiently upper bounds the norm of the difference between the gradients by the norm of the gradients of the loss with respect to the inputs to the last layer. Thus for computational efficiency we restrict to finding the gradients with respect to the weights of the last layer. The outline is sketched in Algorithm \ref{alg:algo_submodular_sampling}.

\begin{algorithm}
\SetAlgoNoLine
 \textbf{Input:} Unlabeled pool $\mathcal{U}$, Total Budget $B$, Samples to label per iteration $k$, Model $M$,  Initial labeled set $\mathcal{L}$; \leavevmode\newline 
 \While{$\vert \mathcal{L} \vert \leq \vert B \vert$}{
    \begin{enumerate}[leftmargin=*,label=$\bullet$]
        \item Train model M on $\mathcal{L}$
        \item $V \xleftarrow[]{} \phi$
        \leavevmode\newline
        \For{$x \in \mathcal{U}$}{
            $e \xleftarrow[]{} \frac{\partial \mathcal{M}(x)}{\partial x}$ \\
            $V = V \cup e$
        }
        $S_0 \xleftarrow[]{} \mathcal{L}$ 
        \item \For{i = 1, 2, ..., k}{
            $e = \underset{e \in V}{argmax}\ F(e \vert S_{i-1})$
            $S_i = S_{i-1} \cup e$
        }
        \item $L \xleftarrow[]{} S$
    \end{enumerate}
 }
 \caption{Greedy Algorithm to sample from the pool of unlabeled data points for Active learning}
  \label{alg:algo_submodular_sampling}
\end{algorithm}

\section{Experimental Setup}\label{sec:experiment}
\subsection{Datasets}
We use 7 standard text classification datasets and their 10 variants as used by \cite{ein-dor-etal-2020-active}. 
Specifically, the datasets used are Wiki Attack \cite{wikiattack/Wulczyn/10.1145/3038912.3052591}, ISEAR \cite{isear/shao/doi:10.1177/0022022114557479}, TREC \cite{trec/li-roth-2002-learning}, CoLA \cite{cola/warstadt-etal-2019-neural} , AG's News \cite{agnews/10.5555/2969239.2969312}, Subjectivity \cite{subjectivity/10.3115/1218955.1218990}, and Polarity \cite{polarity/10.3115/1219840.1219855}.
The experimental setup considers three settings: (1) \textbf{Balanced}, in which the prior probability of a class occurrence is $\geq$ 15\%. Here the initial seed for labeling is obtained by random sampling. (2) \textbf{Imbalanced} and (3) \textbf{Imbalanced practical} in which the prior probability of a  class occurrence is $\leq$ 15\%. The initial seed for labeling is obtained by assuming a high precision algorithm or a query.
For more details on the datasets and experimental setups, we refer the readers to \cite{ein-dor-etal-2020-active} and the Appendix.


\subsection{Comparative Methods}
The acquisition methods are used to query and obtain samples from the unlabeled pool for labeling. In the implementation 25 samples are queried per iteration. We use the active learning acquisition strategies as in \cite{ein-dor-etal-2020-active} namely \textbf{Random}, \textbf{Least Confidence} (\textbf{LC}, \cite{Lewis/DBLP:journals/corr/LewisG94}), \textbf{Monte Carlo Dropout} (\textbf{Dropout}, \cite{pmlr-v48-gal16}), \textbf{Perceptron Ensemble} (\textbf{PE}, \cite{ein-dor-etal-2020-active}), \textbf{Expected Gradient Length} (\textbf{EGL}, \cite{Huang/EGL/DBLP:journals/corr/HuangCRLSC16}) , \textbf{Core-Set} (\cite{sener2018active}), \textbf{Discriminative Active Learning} (\textbf{DAL}, \cite{Gissin/abs-1907-06347}). For details refer to the Appendix.
\subsection{Implementation Details}
The $\text{BERT}_{\text{BASE}}$ model (110 M parameters) is used with a batch size of 50 and a maximum token length of 100 tokens. In each active learning iteration, the model is trained for five epochs from scratch. A learning rate of $5 \times 10^{-5}$ has been used. The other parameters are the same as in the PyTorch implementation of BERT. We run each active learning method for five runs starting from the same initial seed (of 25 samples) for every model for a given run and average the result as in \cite{ein-dor-etal-2020-active}.


\section{Results and Discussion} \label{sec:evaluation}
We aim to answer the below Research Questions:
\begin{enumerate}
    \item \textbf{RQ1}: Is ALLWAS beneficial in the low resource and imbalanced setting?
    \item \textbf{RQ2}: Does the proposed Wasserstein barycentric over-sampling help in the few sample settings compared to the control of no over-sampling?
    \item \textbf{RQ3}: Does the proposed gradient-domain submodular query function perform better than existing approaches in the same space?
    \item \textbf{RQ4}: Is barycentric over-sampling in the wasserstein space significantly better than that in the $\ell_2$ space? 
\end{enumerate}
\subsection{Active Learning Results on the Binary class settings}
The results, for the imbalanced practical binary setting, are shown in the graphs in figures \ref{fig:imb_prac}. The results for the other settings can be found in the Appendix.
For brevity we show the results on the same set of the active learning methods as in \cite{ein-dor-etal-2020-active}. From the Figures in \ref{fig:imb_prac}, we observe that in most of the datasets our method outperforms all the other methods in all settings. 
For the balanced setting, we find that our method performs exceptionally well in the start with lesser data. Then as the training data increases with iterations, the performance of the other methods catches up. Thus we could say that our method converges faster in scenarios where the data is balanced. 
In the two imbalanced settings, we observe an apparent gain in performance. Thus, we conclude that combining submodular query function and barycentric sampling benefits performance in class imbalance cases of active learning (answering \textbf{RQ1}). 




\begin{figure*}[htbp]
  \centering
\setkeys{Gin}{width=0.24\linewidth,height=0.24\linewidth}
  \includegraphics{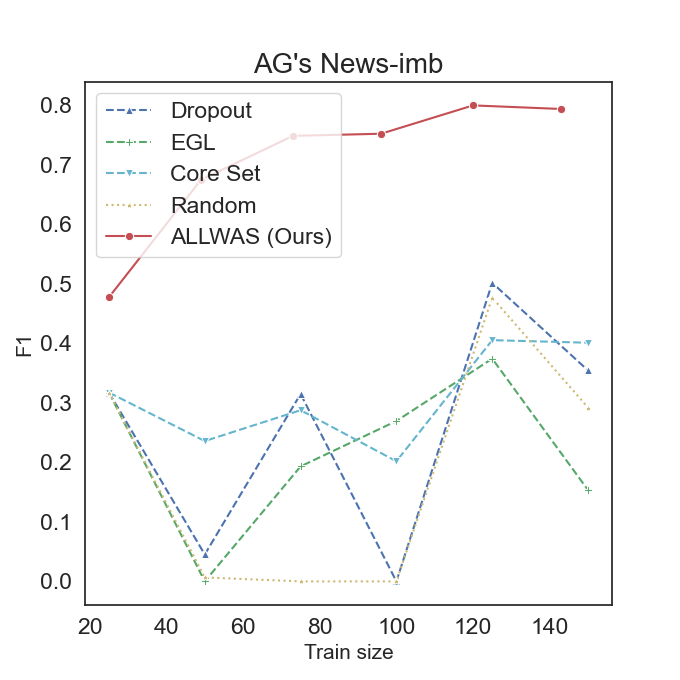}\,%
  \includegraphics{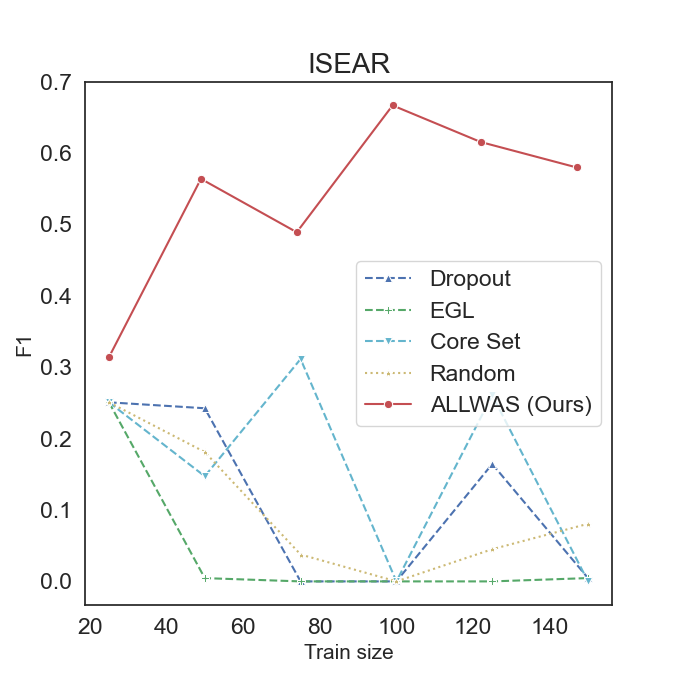}\,%
  \includegraphics{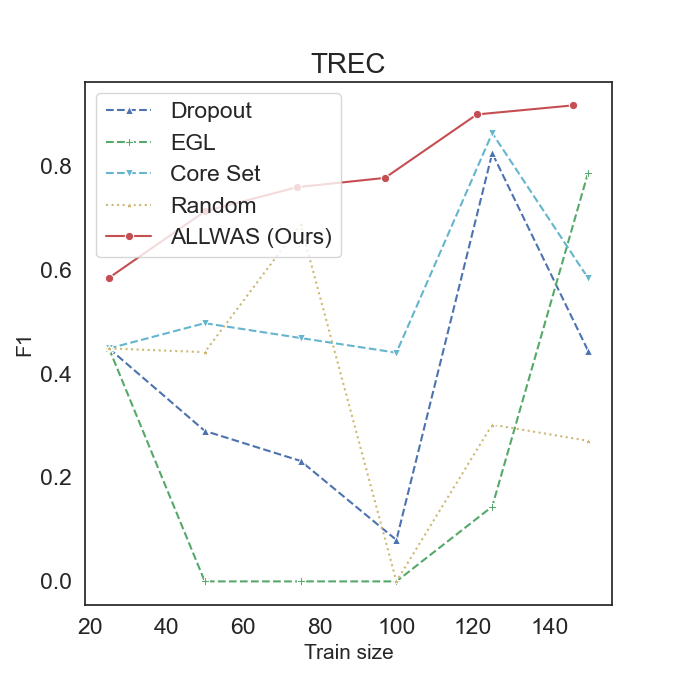}\,%
  \includegraphics{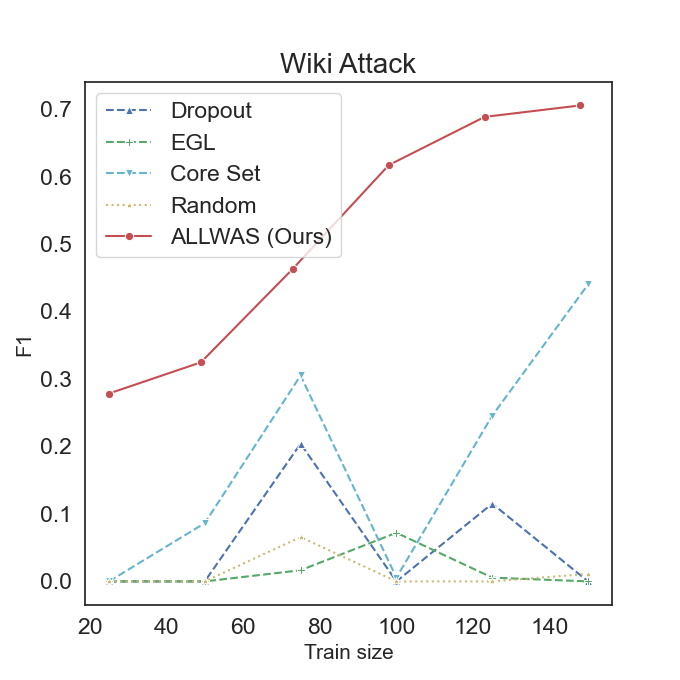}
 \caption{Results on the Imbalanced Practical setting. Our Model clearly outperforms the baselines.}
  \label{fig:imb_prac}
\end{figure*}

\subsection{Few Sample Results}
In the few sample settings, we test the barycentric sampling on a few data points sampled incrementally. The augmentation factor is kept at 20 as a default. The results are plotted in Figure \ref{fig:fsl}. We observe a stark improvement in the results, in some cases the relative increase being as high as 24\%. This shows that in such cases of data scarcity, the task, in this case, classification, could benefit by sampling from the Wasserstein barycenter of the original samples as an augmentation technique independent of the sampling technique used in the query function (answering \textbf{RQ2}). 
\begin{figure*}[htbp]
  \centering
\setkeys{Gin}{width=0.24\linewidth,height=0.24\linewidth}
  \includegraphics{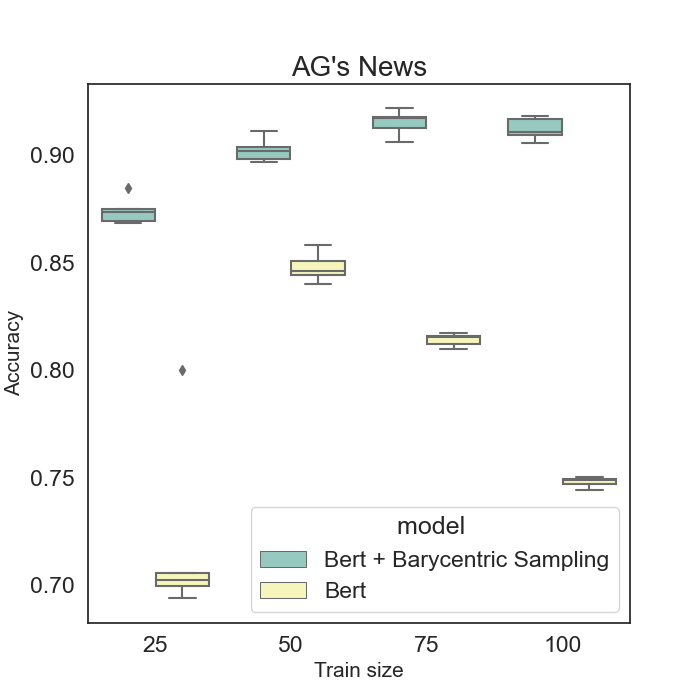}\,%
  \includegraphics{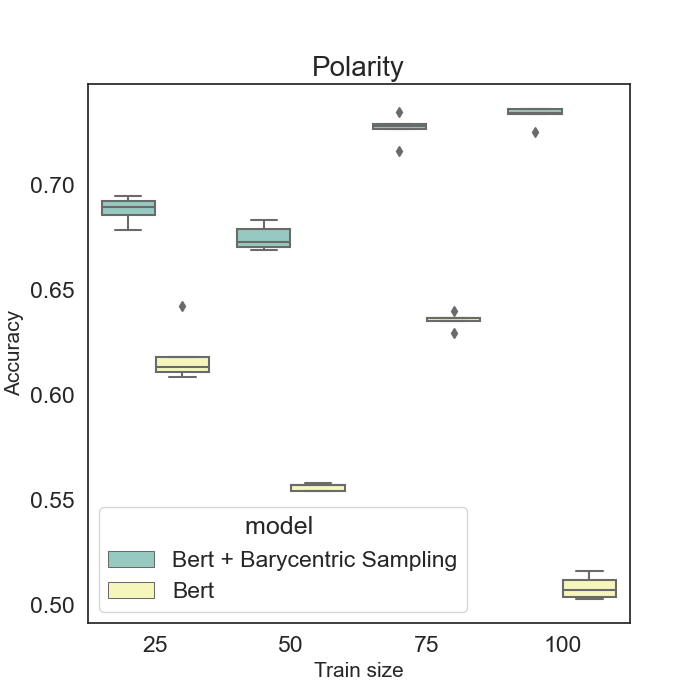}\,%
  \includegraphics{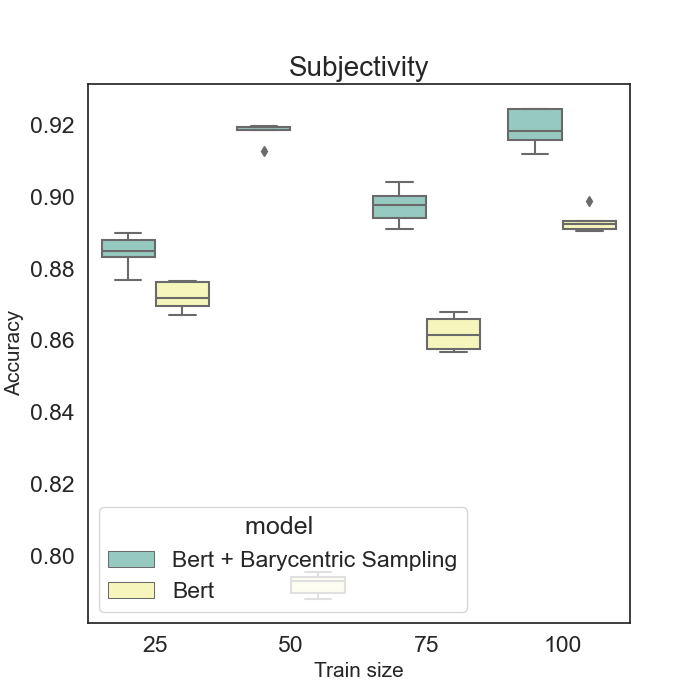}\,%
  \includegraphics{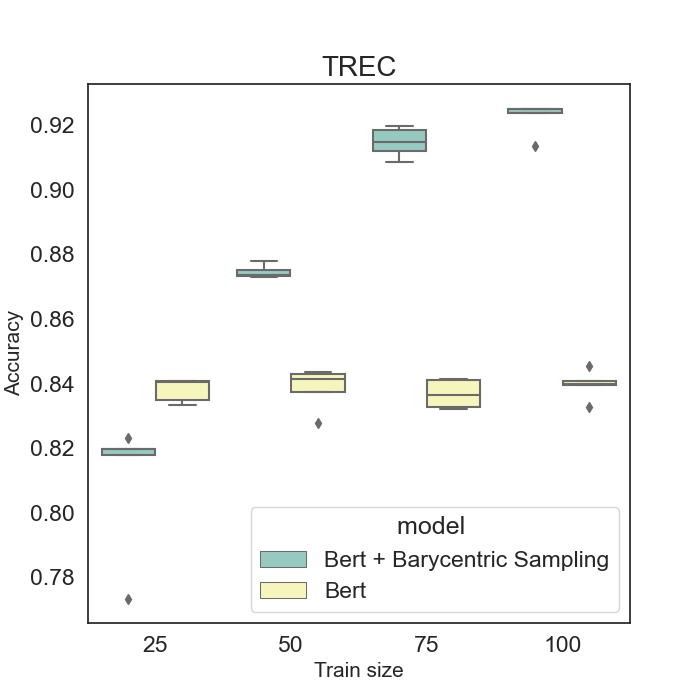}
 \caption{Few sample setting. Our setting (BERT+barycentric sampling) illustrates superior performance. }
  \label{fig:fsl}
\end{figure*}

\begin{figure*}[htbp]
  \centering
\setkeys{Gin}{width=0.24\linewidth,height=0.24\linewidth}
  \includegraphics{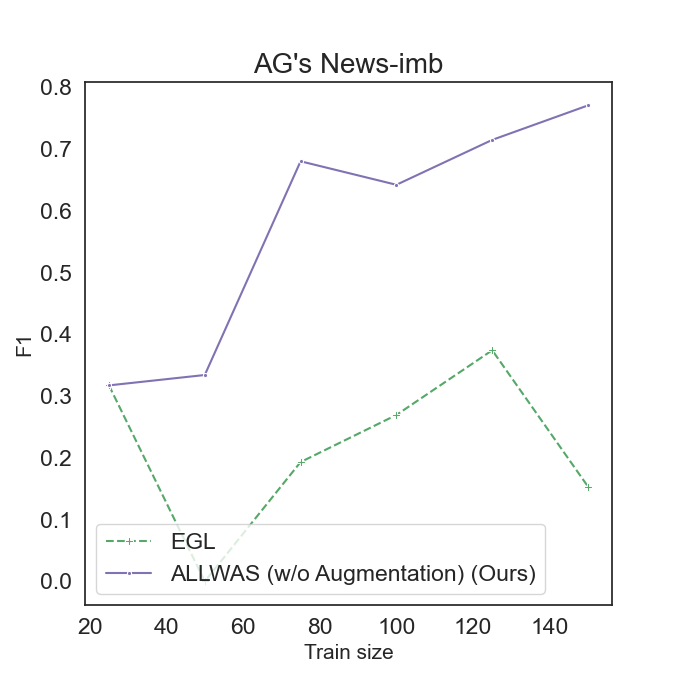}\,%
  \includegraphics{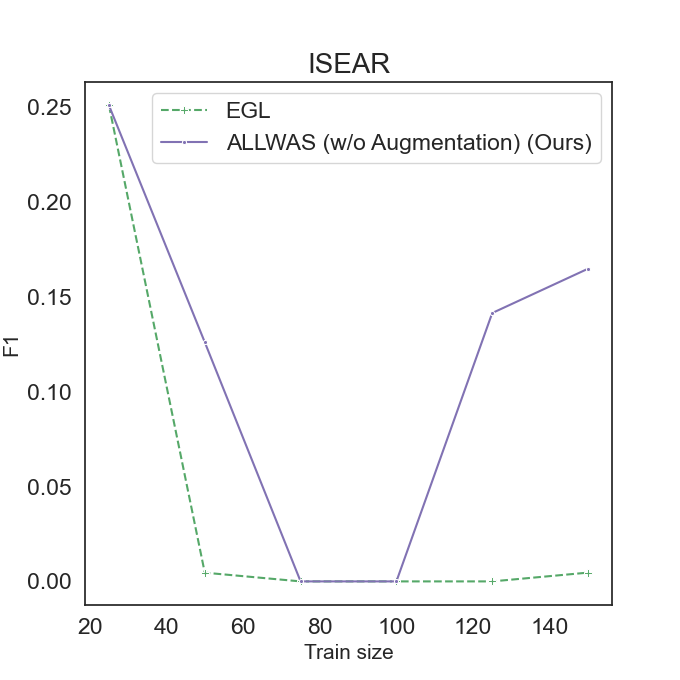}\,%
  \includegraphics{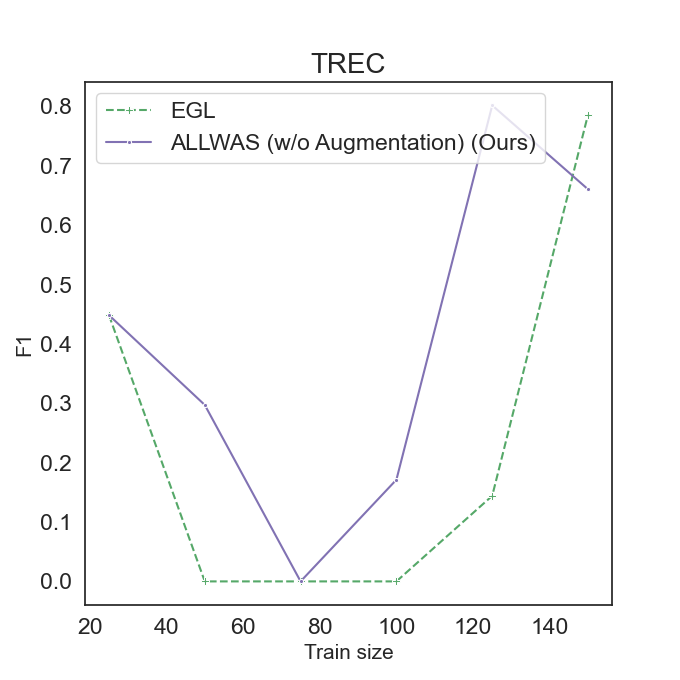}\,%
  \includegraphics{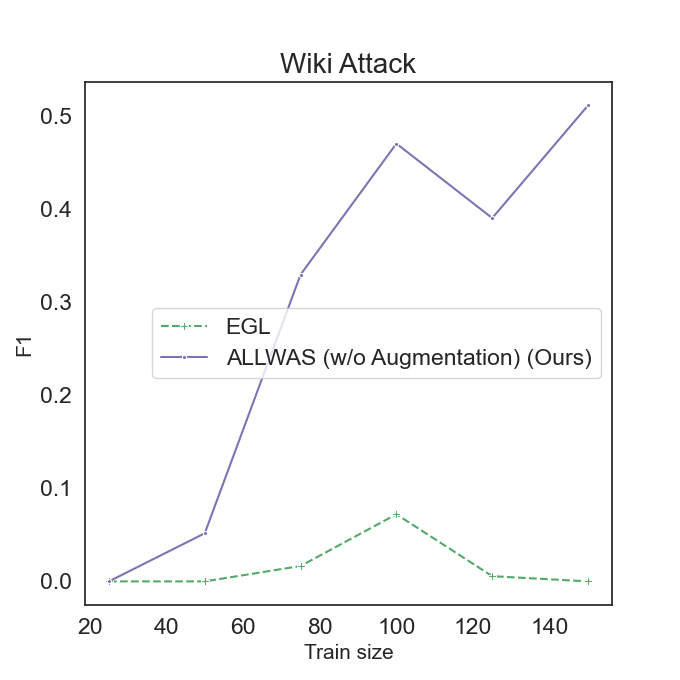}
 \caption{Comparison between selection based on maximum and coreset in the gradient domain}
  \label{fig:ablation_grads}
\end{figure*}

\begin{figure*}[htbp]
  \centering
\setkeys{Gin}{width=0.24\linewidth,height=0.192\linewidth}
  \includegraphics{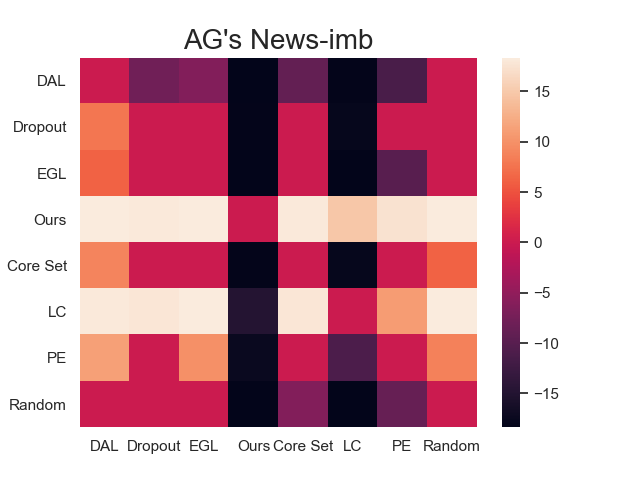}\,%
  \includegraphics{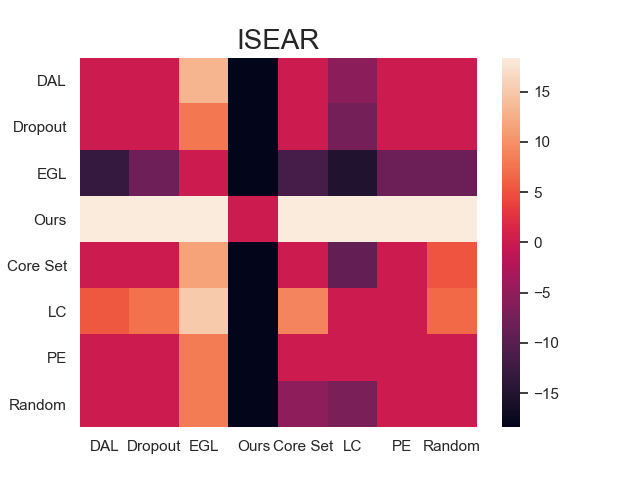}\,%
  \includegraphics{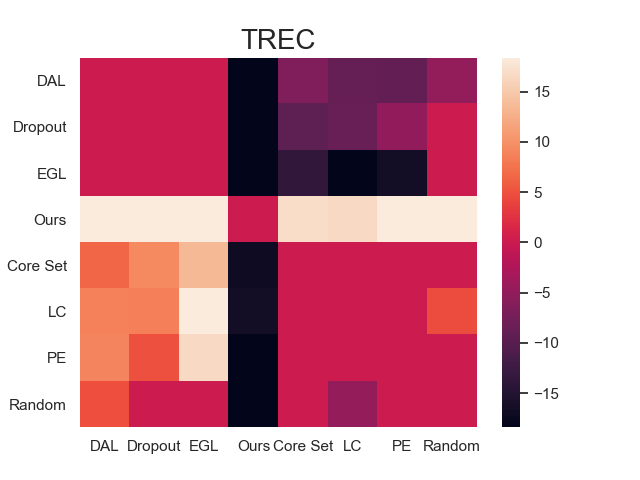}\,%
  \includegraphics{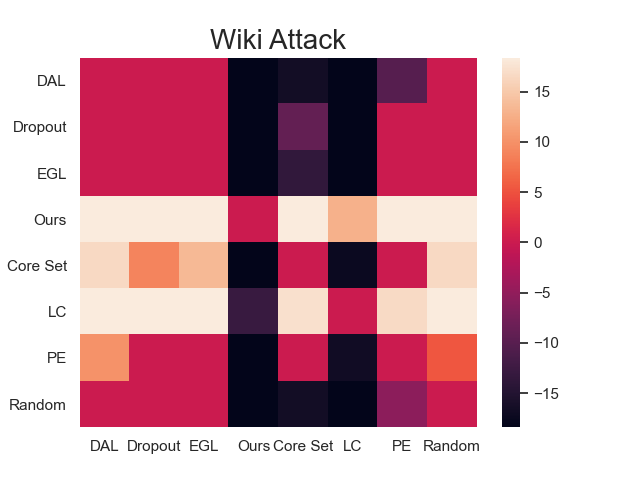}
 \caption{Statistical significance. Our model is statistically significant compared to baseline, illustrating the robustness of our proposed approach.}
  \label{fig:stat_sig}
\end{figure*}

\begin{figure*}[htbp]
  \centering
\setkeys{Gin}{width=0.24\linewidth,height=0.24\linewidth}
  \includegraphics{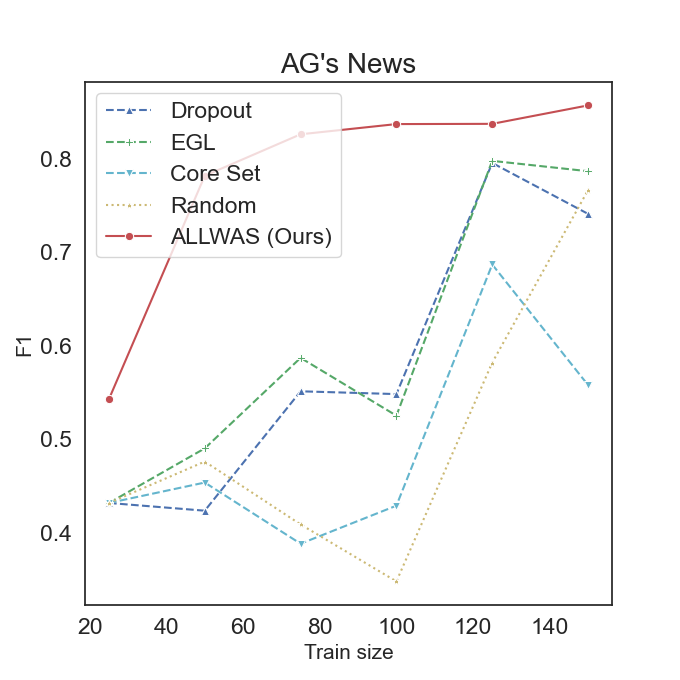}\,%
  \includegraphics{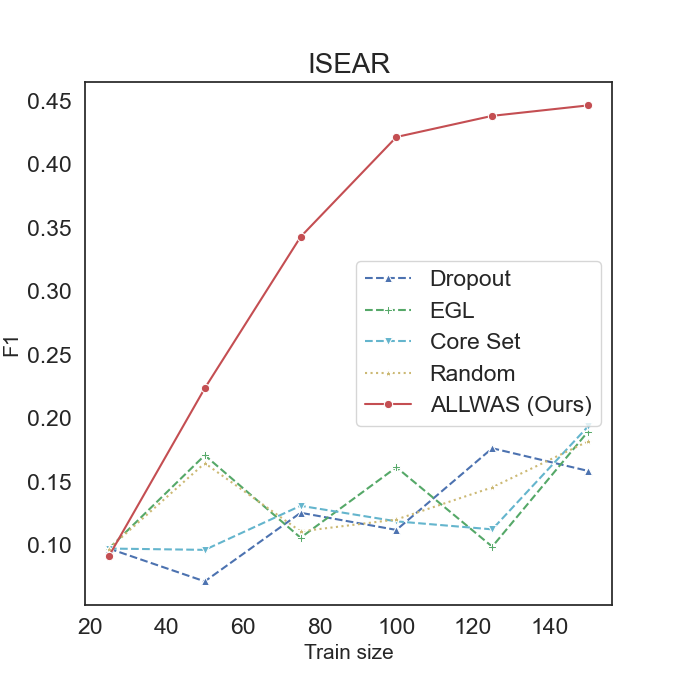}\,%
  \includegraphics{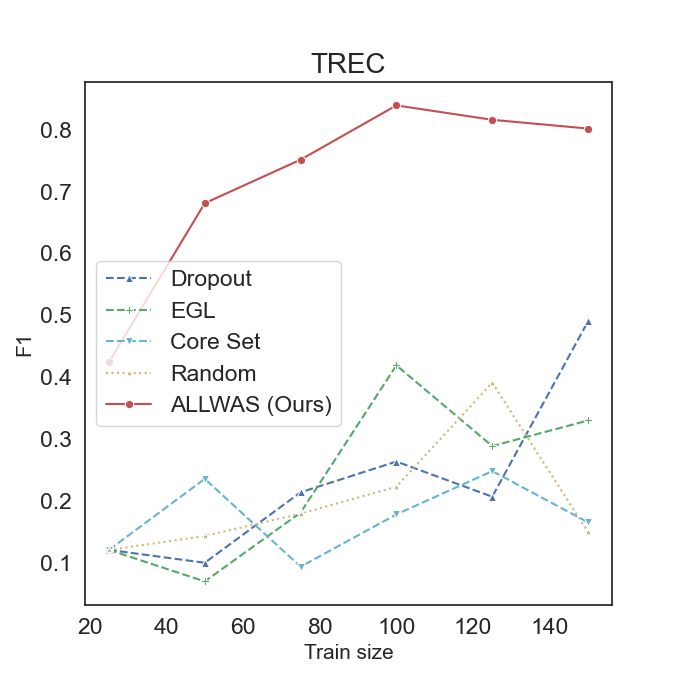}\,%
  \includegraphics{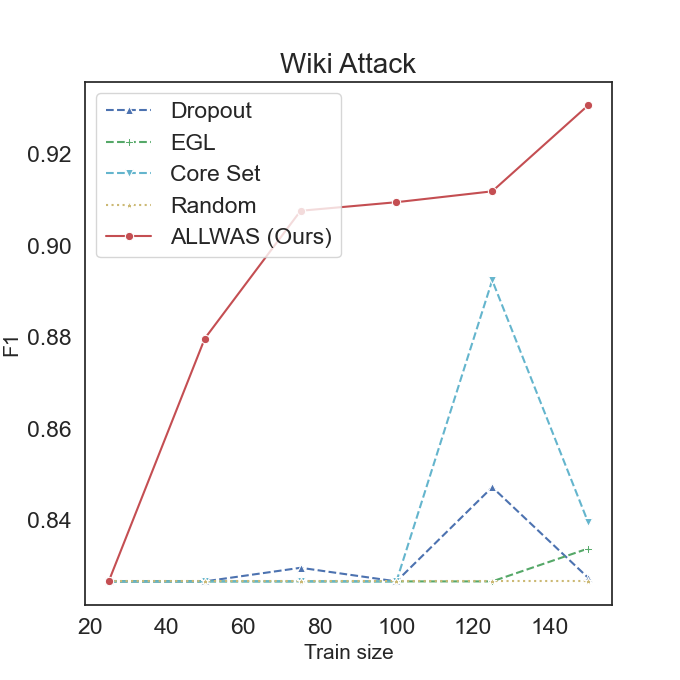}
 \caption{Results on the Multiclass setting. Our model performs significantly better than baselines.}
  \label{fig:multilabel}
\end{figure*}

  

\subsection{Coreset vs Maximum Gradient}
In this section, we study the performance of our query function (ALLWAS w/o Augmentation) that selects a coreset of the gradients against the expected gradient length method that picks samples with the largest expected gradient magnitudes (EGL, \cite{Huang/EGL/DBLP:journals/corr/HuangCRLSC16}). The plots in figure \ref{fig:ablation_grads} show that our query sampling technique outperforms the EGL method. Thus we conclude that selecting core-sets is the better approach against picking samples with extreme values in the gradient domain in assertion of \textbf{RQ3}.



\subsection{Sampling from Wasserstein vs $\ell_2$ barycenter} \label{l2_vs_wass_ablation}
In order to confirm the claim made in subsection \ref{OT_wass}, we perform a comparison between $\ell_2$ barycentric and Wasserstein barycentric over-sampling on the imbalanced practical setting. Sampling from $\ell_2$ is done by performing a kernel density estimation in the embedding space and then sampling from the resulting distribution. We report the Wilcoxon signed-rank test statistics in table \ref{tab:tab-l2_vs_wass} with Bonferroni correction to take into account the runs from all the datasets, settings, and iterations. Statistically significant (better) results are reported of the two sampling techniques against each other and a control of no over-sampling (augmentation). 
The results indicate that while both the over-sampling methods perform better than no sampling, reaffirming \textbf{RQ2}, sampling from the Wasserstein barycenter performs better than sampling from the $\ell_2$ barycenter confirming the claim in \ref{OT_wass} and asserting \textbf{RQ4}.
\begin{table}[ht!]
\small
    \begin{tabular}{p{2cm}|p{2cm}|p{2cm}}
     \toprule
      Significance wrt & $\ell_2$ & Wasserstein \\
     \midrule
     No over-sampling & $<10^{-20}$& $<10^{-26}$\\
     $\ell_2$ & --& $<10^{-12}$\\
     Wasserstein & --& -- \\
    \bottomrule
     
    \end{tabular}
\caption{Comparison of over-sampling (for augmentation) from the $\ell_2$ vs Wasserstein barycenters, indicating p-values if the method of the column is significantly better than that of the row. -- indicates statistically insignificant or worse performance}
\label{tab:tab-l2_vs_wass}
\end{table}

\subsection{Statistical Significance}
We study if the performance of our methods is statistically significant concerning the baselines for each dataset. We perform the Wilcoxon signed-rank test for significance with Bonferroni correction. We select the Wilcoxon test due to its nonparametric nature. The significant results in the form of heatmaps of the logarithms of the p-values are shown in figure \ref{fig:stat_sig}. The insignificant results have their values at 0. From the heatmaps, our method outperforms the baselines in all datasets, indicating the increase is indeed statistically significant. The results echo the observation made by \cite{ein-dor-etal-2020-active} that no single sampling strategy is better than all others. However, in the low data regime that we operate in, many of the methods are not significantly better than the random sampling baseline.

\subsection{Multi class Active Learning Results}
Similar to the binary settings, we also study the performance of our method in the multi-class setting.
We find that in the multi-class setting, too, our method works better than the baselines, as can be seen in figure \ref{fig:multilabel}. This shows that our method is not restricted to the binary classification setting but also to the more generic multi-class cases.

\subsection{Effect of augmentation factor}\label{ablation_aug_main}
We study the effect of the multiplicative factor while augmenting the samples using barycentric sampling technique. Here the number of samples of which to compute the barycenter is kept at two. The results are given in the appendix. It is observed that as the augmentation factor is increased, the performance increases initially when the data is low. However, as more data is acquired from the unlabeled pool, the gap reduces. This indicates that we may benefit more by keeping the augmentation factor high in the low data regime.



\subsection{Effect of number of samples to find the barycenter}
Similar to subsection \ref{ablation_aug_main}, we study the effect of the number of data points used to find the barycenter. Keeping the augmentation factor fixed at 20, we vary the number of samples to find the barycenter. Results are in the Appendix. It is observed that as the data points to the sample increases, the performance marginally drops. This becomes intuitive if we think of computing the barycenter as averaging over the samples. If we average out many samples, we effectively get the representative sample which would be similar in most iterations. 


\section{Conclusion} \label{sec:conclusion}
This paper presents and studies novel approaches of data sampling using concepts from submodular optimization and optimal transport theory for active learning in language models. We find that augmenting data using the Wasserstein barycenter helps to learn in the few sample setting. Further, we conclude that using a submodular function based on the Wasserstein distance for sampling in the gradient domain helps in active learning. Future works could explore data subset distances using optimal transport to find the subset of data that would benefit the model. It also remains to be explored if using core-sets obtained in this manner would help speed up the training of language models without affecting its accuracy by a large margin.  We point readers to the open questions in this domain as next viable steps.

\bibliography{emnlp2018}

\appendix
\section{Appendix}
\subsection{Details of Dataset}
We use seven standard text classification datasets and their ten variants as used by \cite{ein-dor-etal-2020-active}. 
Specifically, the datasets used are Wiki Attack \cite{wikiattack/Wulczyn/10.1145/3038912.3052591} which annotates wikipedia discussions for offensive content, ISEAR \cite{isear/shao/doi:10.1177/0022022114557479} which reports for personal accounts of emotions, TREC \cite{isear/shao/doi:10.1177/0022022114557479} which classifies question categories, CoLA \cite{cola/warstadt-etal-2019-neural} which identifies the content for linguistic acceptability, AG's News \cite{agnews/10.5555/2969239.2969312} which categorises news articles, Subjectivity \cite{subjectivity/10.3115/1218955.1218990} which classifies movie snippets into subjective and objective and Polarity \cite{polarity/10.3115/1219840.1219855} which provides sentiment categories of movie reviews. 
The datasets which contain labels with a prior of greater than 15\% are taken into the balanced setting and those with less than a 15\% prior are considered in the imbalanced setting as in \cite{ein-dor-etal-2020-active}. The experimental setup considers 3 settings: (1) \textbf{Balanced}, in which the prior probability of a class occurrence is $\geq$ 15\%. Here the initial seed for labeling is obtained by random sampling. (2) \textbf{Imbalanced} and (3) \textbf{Imbalanced practical} in which the prior probability of a  class occurrence is $\leq$ 15\%.
In the case of the Imbalanced setting the initial seed is taken by randomly sampling from the class with the low prior. Here the assumption is that there exists a heuristic to obtain an unbiased sample set with high precision of the low prior class. As this may not always hold true the Imbalanced practical setting samples using a simple and empirical heuristic such as a query based search for the samples belonging to the low prior class. This gives a (biased) set of samples of the class which are then used for labeling . For the class with a high prior probability random samples are drawn from the dataset and are labeled as such for both the imbalanced settings. 

\begin{table}[]
    \centering
    \begin{tabular}{c c c c c}
        \toprule 
        \textbf{No.}& \textbf{Dataset}& \textbf{Size}& \textbf{Class}& \textbf{Prior} \\
        \midrule 
        1& Subjectivity-imb& 5,556& subjective& 10\% \\
        2& Polarity-imb& 5,923& positive& 10\% \\
        3& AG’s News-imb& 17,538& world& 10\% \\
        4& Wiki attack& 21,000& general& 12\% \\
        5& ISEAR& 7,666& fear& 14\% \\
        6& TREC& 5,952& location& 15\% \\
        7& AG’s News& 21,000& world& 25\% \\
        8& CoLA& 9,594& unacceptable& 30\% \\
        9& Subjectivity& 10,000& subjective& 50\% \\
        10& Polarity& 10,662& positive& 50\% \\
        \bottomrule 
    \end{tabular}
    \caption{Dataset Statistics}
    \label{tab:data_table}
\end{table}

\subsection{Details of Comparative Methods}
The acquisition methods are used to query and obtain samples from the unlabeled pool for labeling. In the implementation 25 samples are queried per iteration. We use the active learning acquisition strategies as in \cite{ein-dor-etal-2020-active} 
as below:
\begin{enumerate}
    \item \textbf{Random}: The data for labeling are randomly sampled from the unlabeled pool.
    \item \textbf{Least Confidence} (\textbf{LC}, \cite{Lewis/DBLP:journals/corr/LewisG94}): This method picks the top $k$ samples for which the model uncertainty is the highest.
    \item \textbf{Monte Carlo Dropout} (\textbf{Dropout}, \cite{pmlr-v48-gal16}): This uses Monte Carlo dropout during inference for multiple runs and averages the probabilities followed by sampling the least certain instances.
    \item \textbf{Perceptron Ensemble} (\textbf{PE}, \cite{ein-dor-etal-2020-active}): Here the output of an ensemble of models is used to pick the instances with highest uncertainty. To avoid the computational cost associated with training an ensemble of BERT models, this method uses the perceptron models trained on the CLS output of the finetuned BERT.
    \item \textbf{Expected Gradient Length} (\textbf{EGL}, \cite{Huang/EGL/DBLP:journals/corr/HuangCRLSC16}): The samples are selected based on the largest expected gradient norm as in \cite{Huang/EGL/DBLP:journals/corr/HuangCRLSC16}. The expectation is over the model predicted probabilities.
    \item \textbf{Core-Set} (\cite{sener2018active}): This method picks samples that best cover the dataset in the embedding space (CLS) using the greedy method desribed in \cite{sener2018active}.
    \item \textbf{Discriminative Active Learning} (\textbf{DAL}, \cite{Gissin/abs-1907-06347}): This technique selects samples that make the L most representative instances of the entire pool as per \cite{Gissin/abs-1907-06347}.
\end{enumerate}



\subsection{Additional Results} \label{sec:evaluation}
\subsubsection{Active Learning Results on the Binary class settings}
The results for the three binary settings are shown in the graphs in figures \ref{fig:bal}, \ref{fig:imb} and \ref{fig:imb_prac}. From the figures we observe that in most of the datasets our method outperforms all the other methods in all settings. 
We report the $f_1$ scores for all settings, since in the balanced case also there may be a slight class imbalance (upto $60\%$). For the balanced setting we find that our method performs exceptionally well in the start with lesser data and then as the training data increases with iterations the performance of the other methods catch up. Thus we could say that our method converges faster in scenarios where the data is balanced. There was one exception with the Cola dataset in which the metric drops as compared to others. Upon further investigating we find that the upsampling causes a drop in performance in this case. Thus, while sampling in this manner may cause an increase in performance in most of the cases it may require the practitioner to fine tune the factor by which to augment the data. In the other 2 settings, namely the imbalanced settings, we observe a clear gain in performance. Thus we conclude that the combination of our submodular query function and barycentric sampling benefits performance in active learning scenarios where there is prevalence of class imbalance. 

\begin{figure*}[htbp]
  \centering
\setkeys{Gin}{width=0.24\linewidth,height=0.24\linewidth}
  \includegraphics{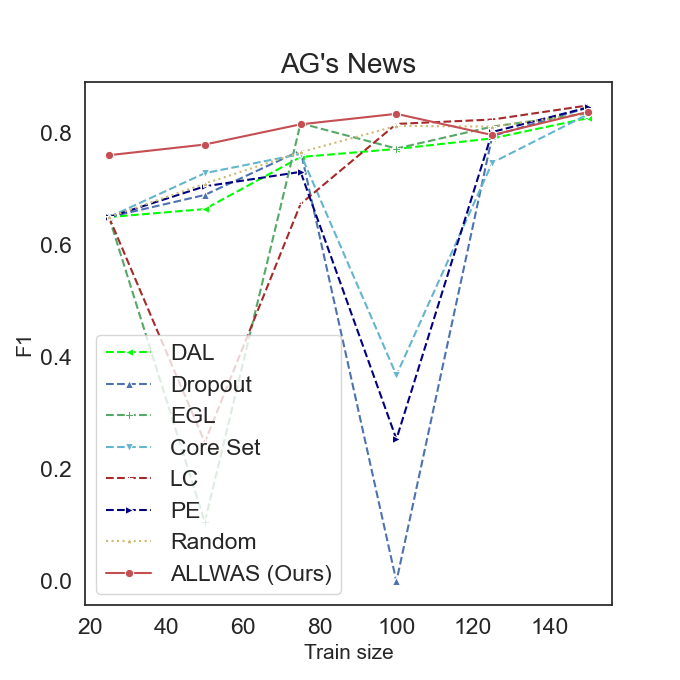}\,%
  \includegraphics{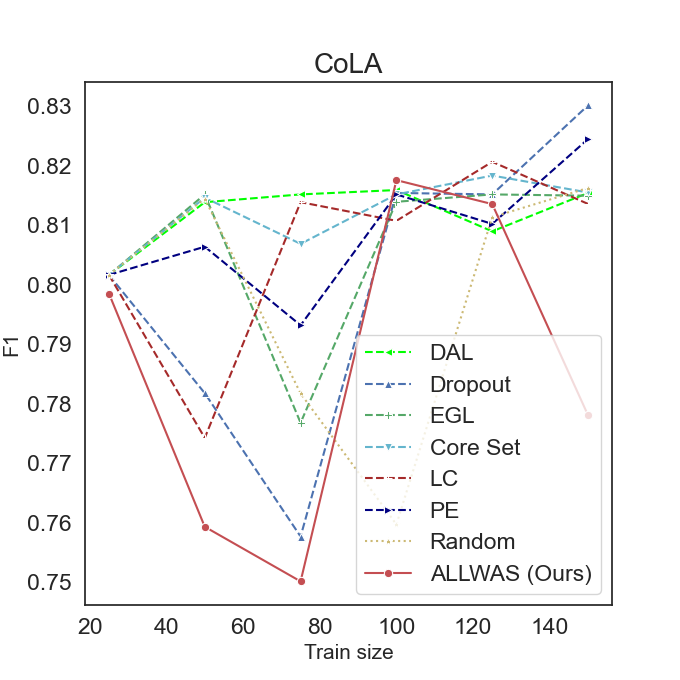}\,%
  \includegraphics{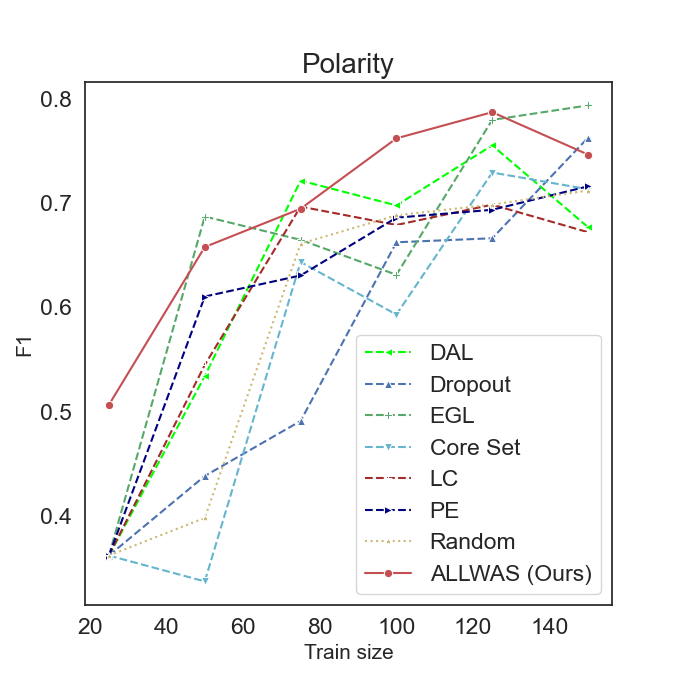}\,%
  \includegraphics{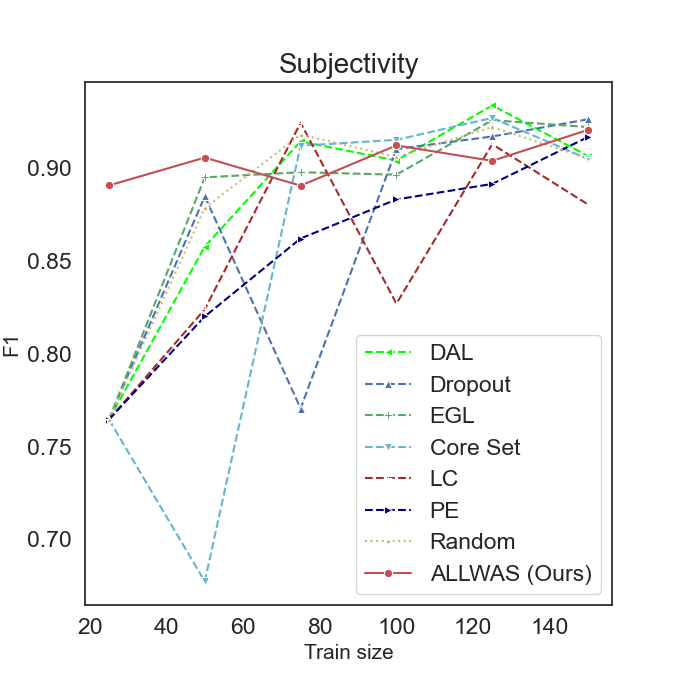}
 \caption{Results on the Balanced setting}
  \label{fig:bal}
\end{figure*}

\begin{figure*}[htbp]
  \centering
\setkeys{Gin}{width=0.32\linewidth,height=0.32\linewidth}
  \includegraphics{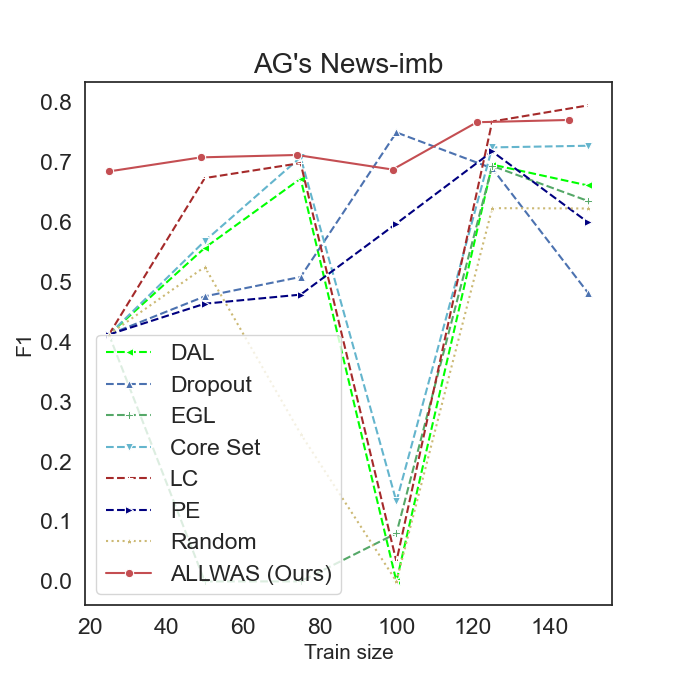}\,%
  \includegraphics{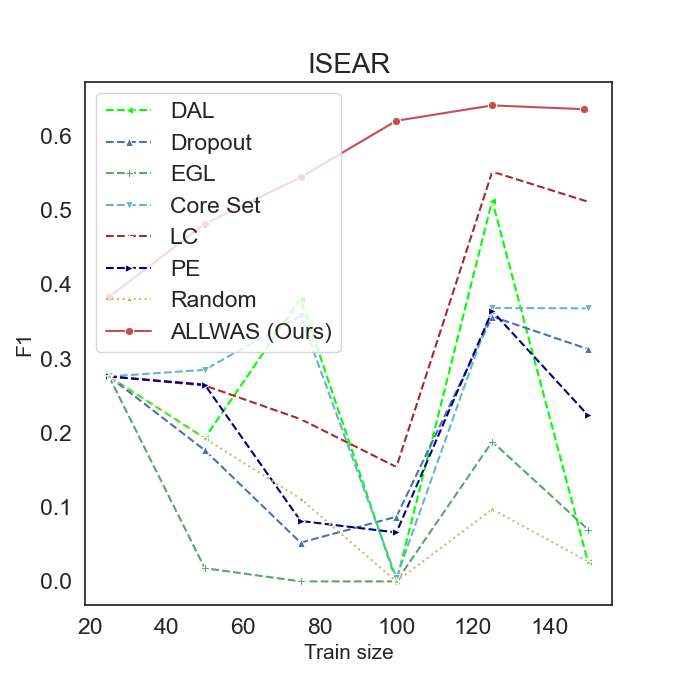}\,%
  \includegraphics{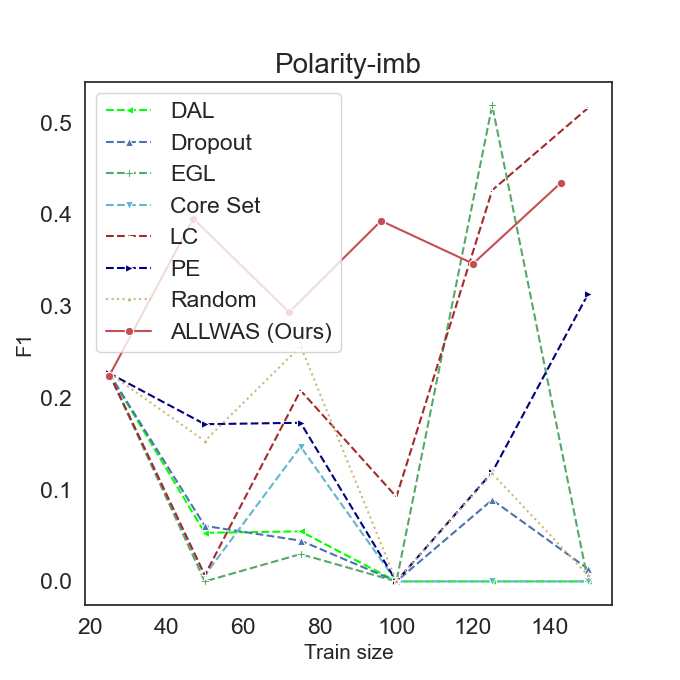}

  \includegraphics{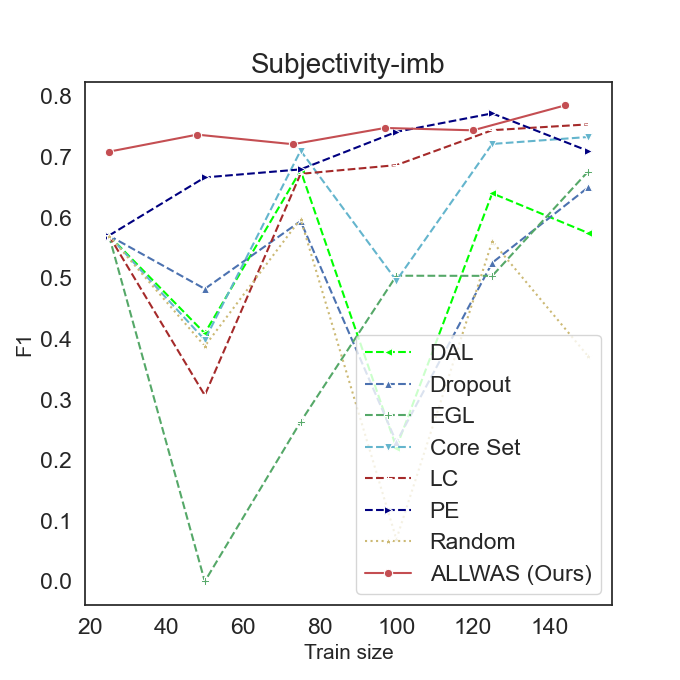}\,%
  \includegraphics{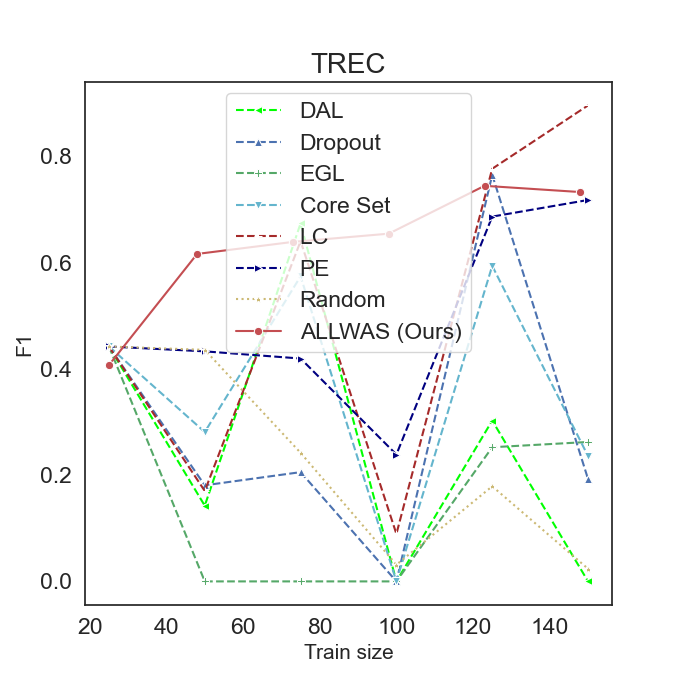}\,%
  \includegraphics{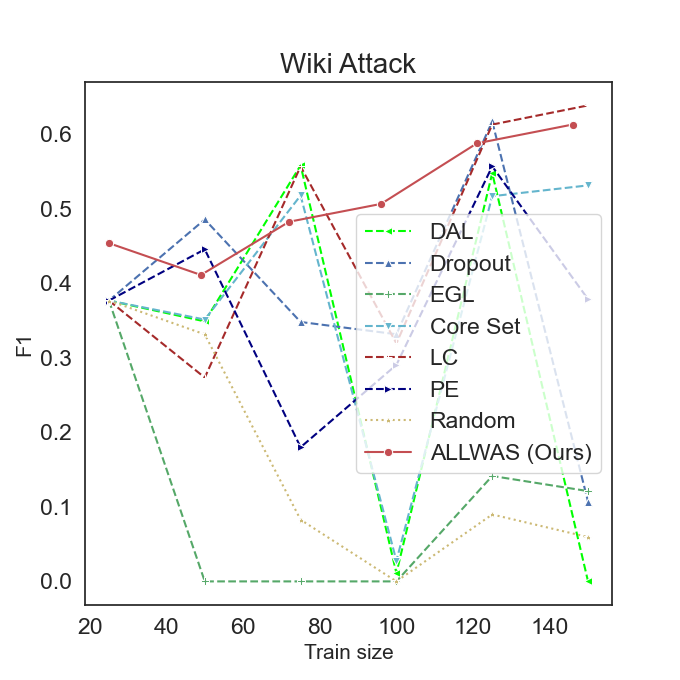}
 \caption{Results on the Imbalanced setting}
  \label{fig:imb}
\end{figure*}

\begin{figure*}[htbp]
  \centering
\setkeys{Gin}{width=0.24\linewidth,height=0.24\linewidth}
  \includegraphics{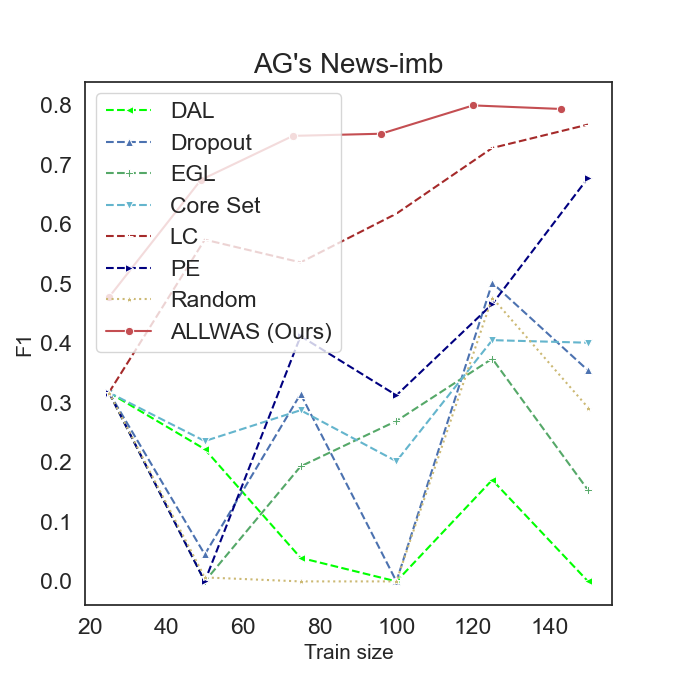}\,%
  \includegraphics{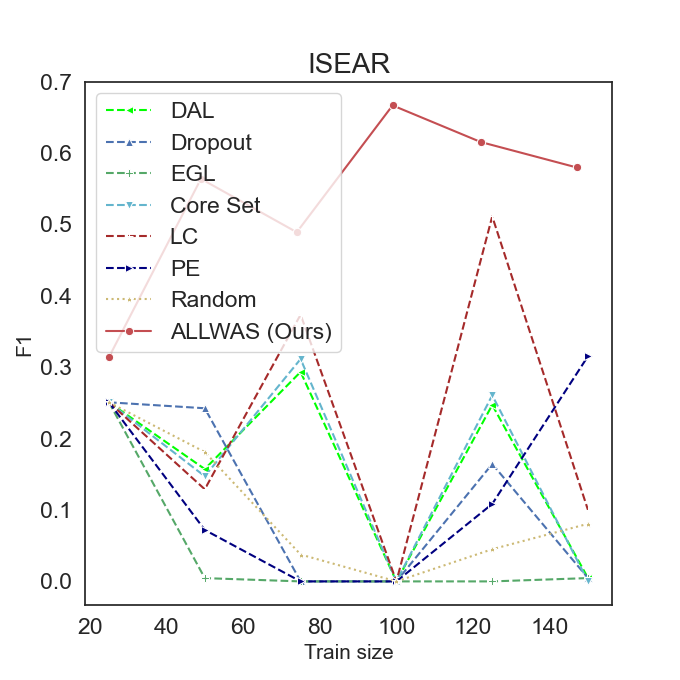}\,%
  \includegraphics{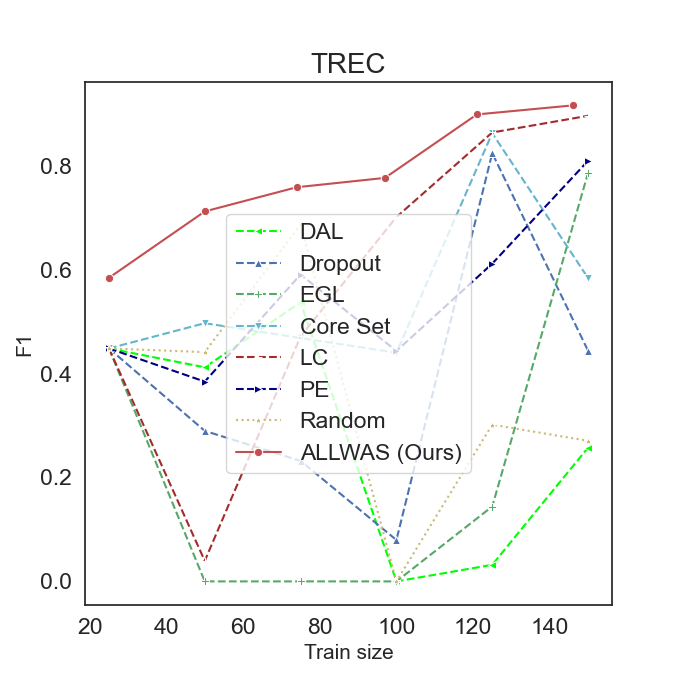}\,%
  \includegraphics{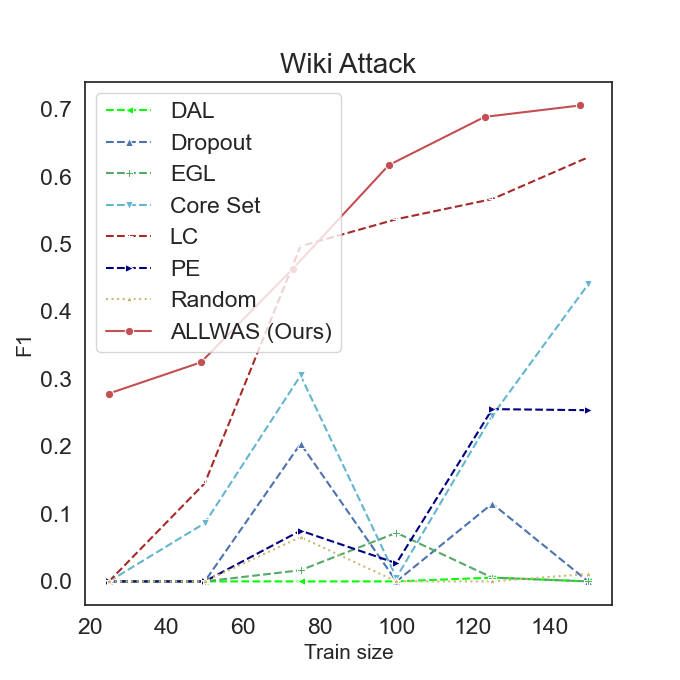}
 \caption{Results on the Imbalanced Practical setting}
  \label{fig:imb_prac}
\end{figure*}

\begin{figure*}[htbp]
  \centering
\setkeys{Gin}{width=0.24\linewidth,height=0.24\linewidth}
  \includegraphics{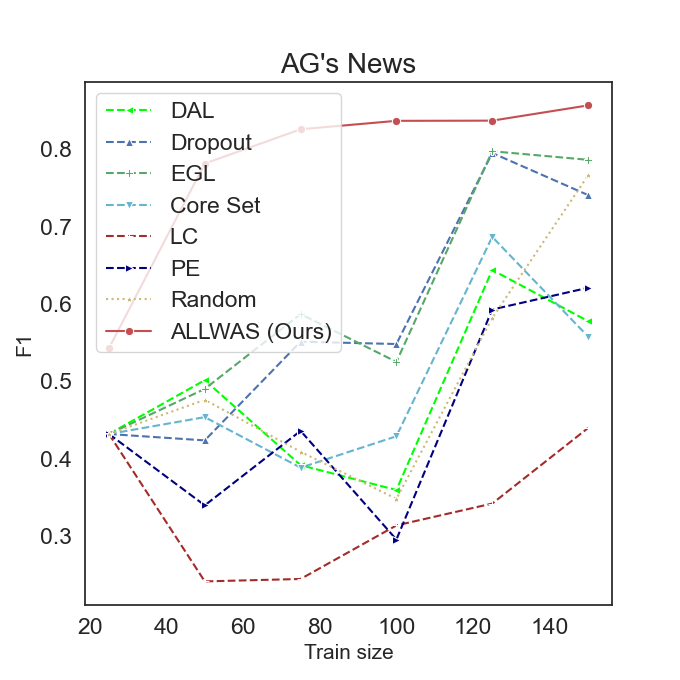}\,%
  \includegraphics{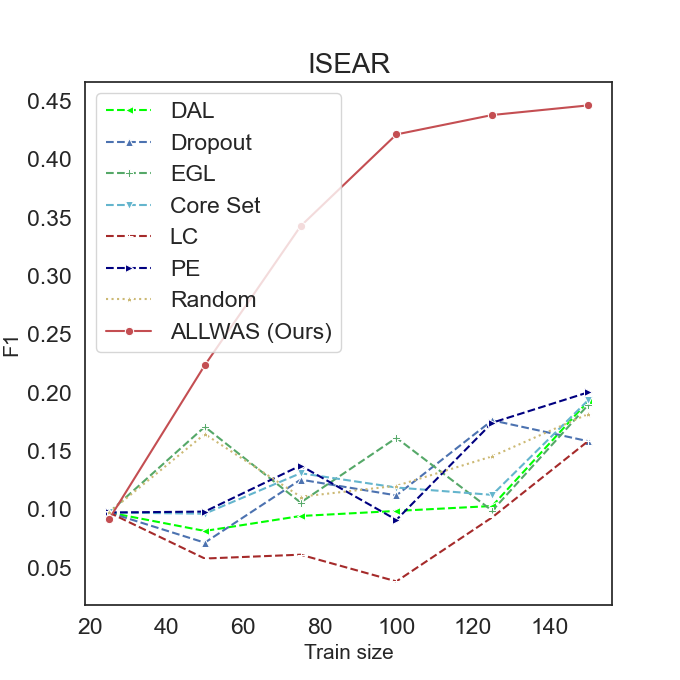}\,%
  \includegraphics{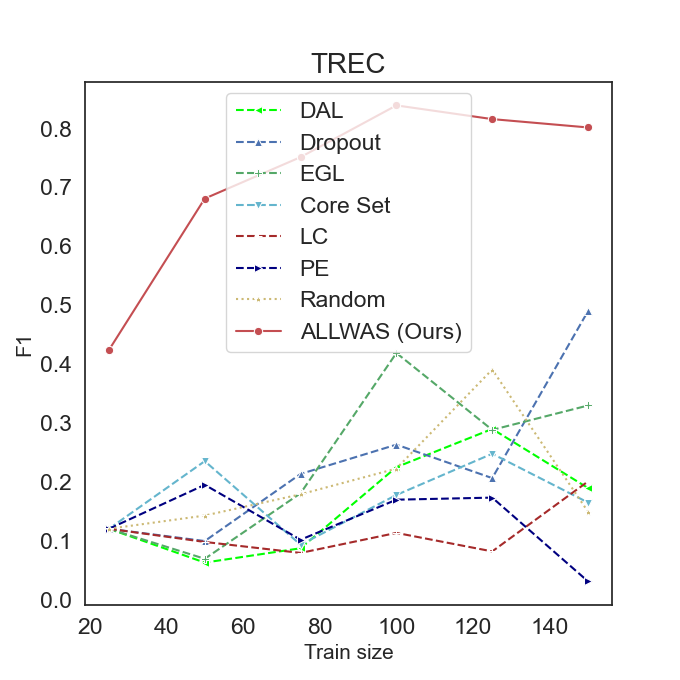}\,%
  \includegraphics{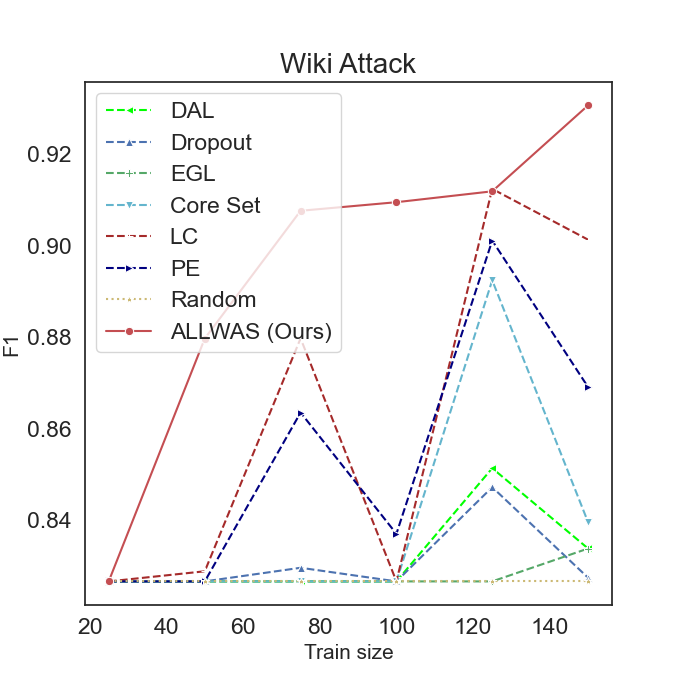}
 \caption{Results in the Multi-class setting}
  \label{fig:multilabel}
\end{figure*}

\subsubsection{Multi class Active Learning Results}
In addition to the results in the main paper, here we report the results on all methods for the multi-class setting in figure \ref{fig:multilabel}.

\subsubsection{Effect of augmentation factor}\label{ablation_aug}
We would like to study the effect of the multiplicative factor while augmenting the samples using barycentric sampling technique. Here the number of samples of which to compute the barycenter are kept at 2. The results are shown in figure \ref{fig:ablation_aug}. It can be seen that as the augmentation factor is increased the performance increases initially, when the data is low, but as more data is acquired from the unlabeled pool the gap reduces and also reverses. This indicates that we may benefit more by keeping the augmentation factor high in the low data regime.
\begin{figure*}[htbp]
  \centering
\setkeys{Gin}{width=0.24\linewidth,height=0.24\linewidth}
  \includegraphics{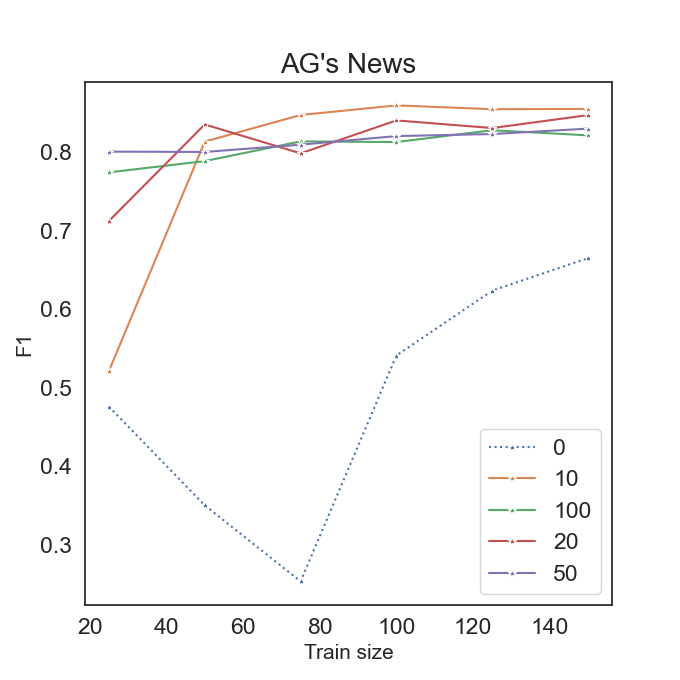}\,%
  \includegraphics{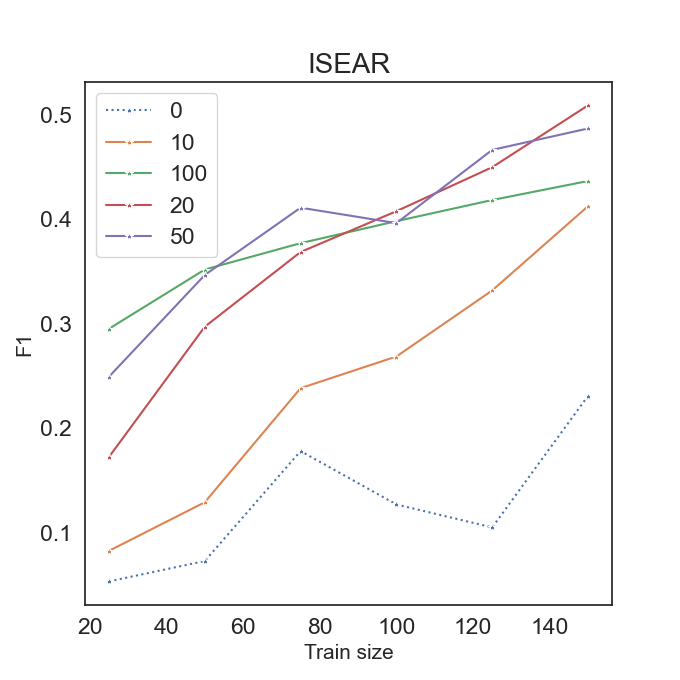}\,%
  \includegraphics{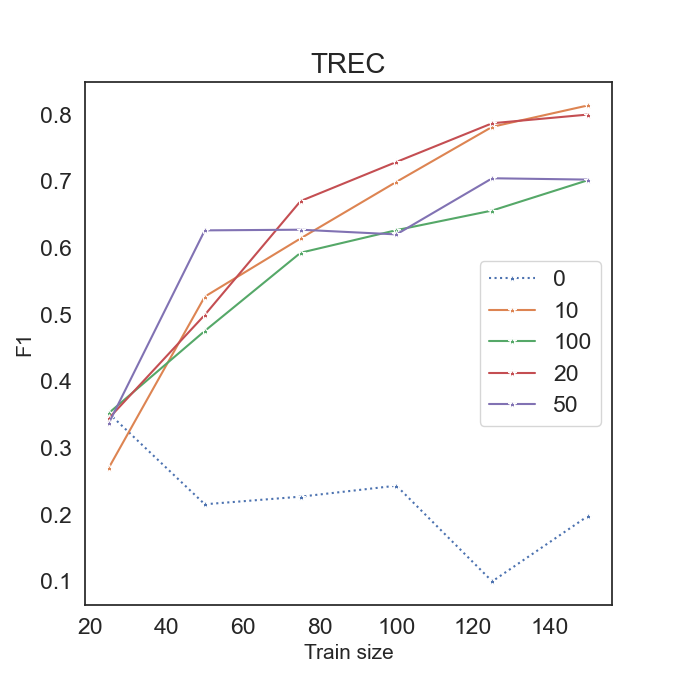}\,%
  \includegraphics{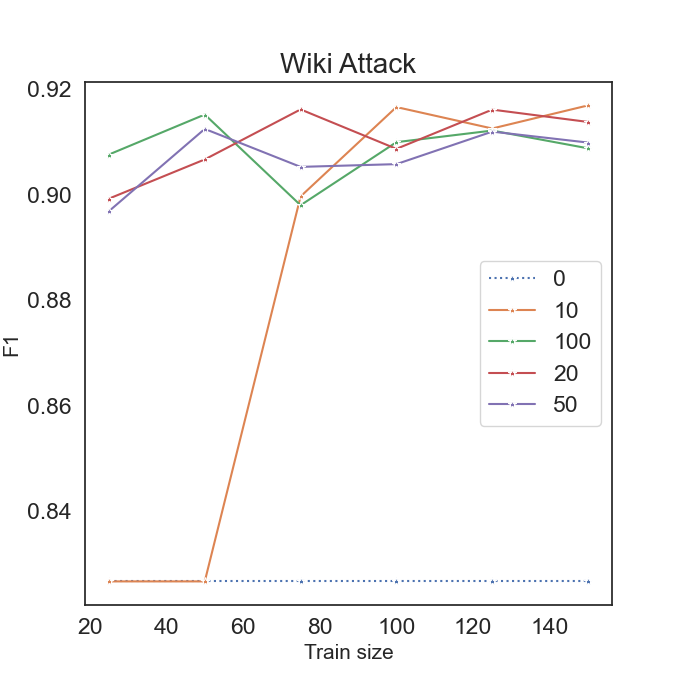}
 \caption{Ablation study on the augmentation factor}
  \label{fig:ablation_aug}
\end{figure*}

\begin{figure*}[htbp]
  \centering
\setkeys{Gin}{width=0.24\linewidth,height=0.24\linewidth}
  \includegraphics{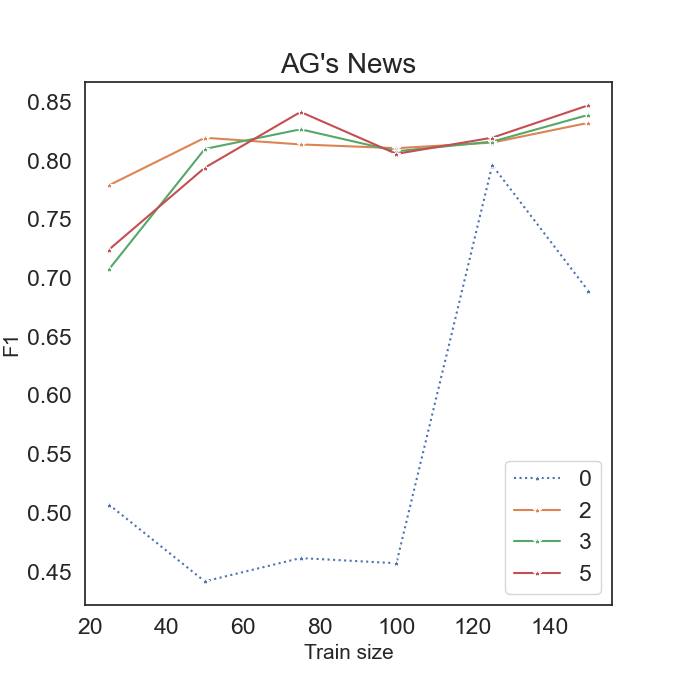}\,%
  \includegraphics{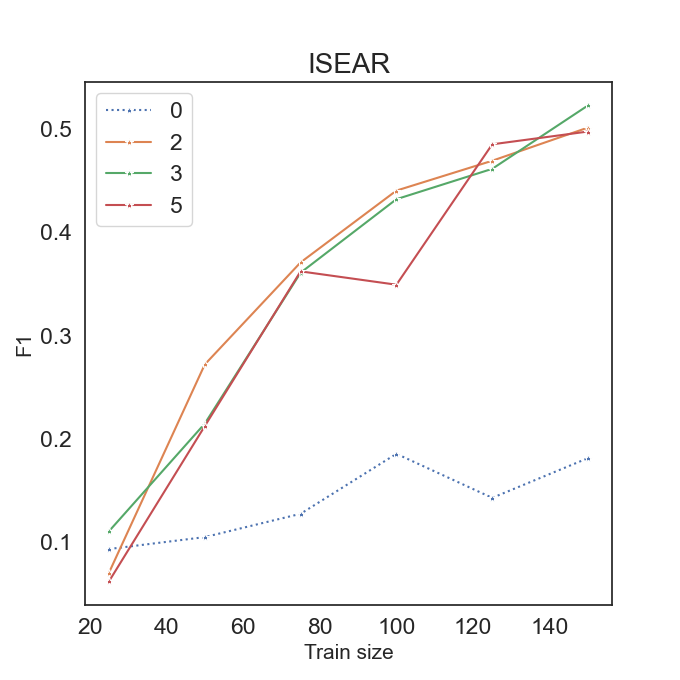}\,%
  \includegraphics{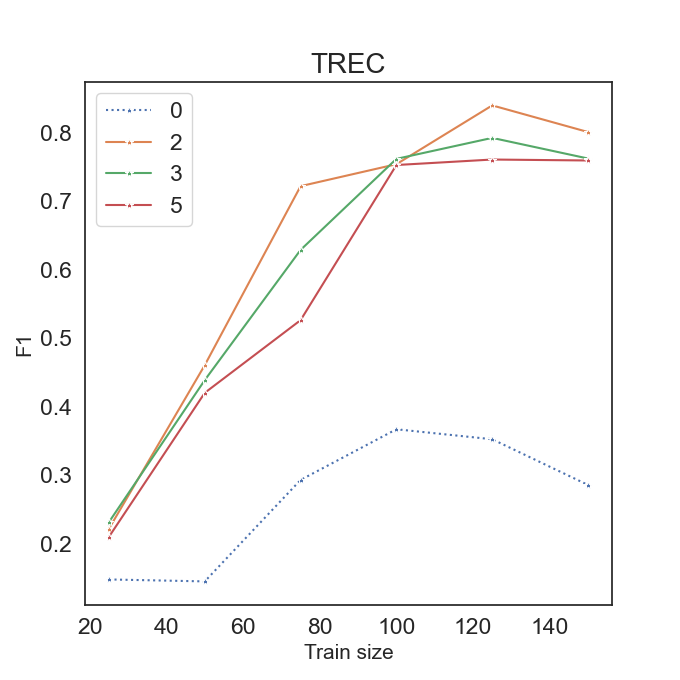}\,%
  \includegraphics{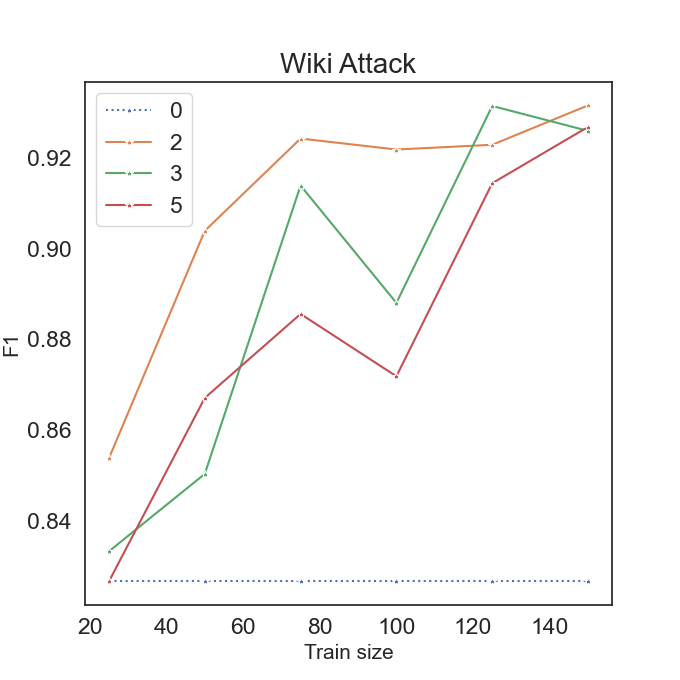}
 \caption{Ablation study on the number of samples to compute the barycenter}
  \label{fig:ablation_bary}
\end{figure*}

\begin{figure}[htbp]
\setkeys{Gin}{width=1.0\linewidth}
  \includegraphics{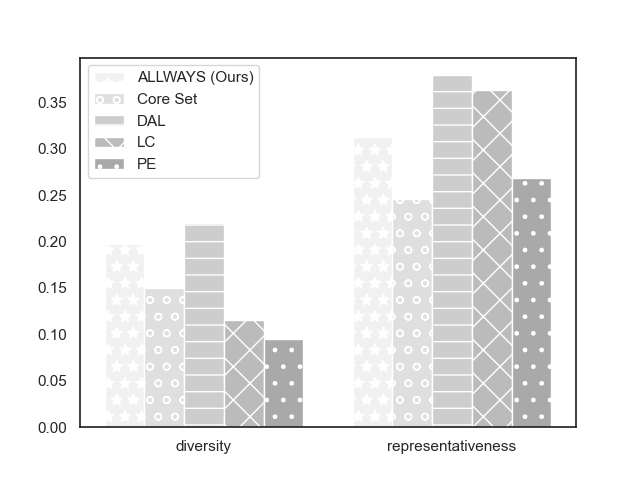}
  \caption{Diversity and Representativeness}
  \label{fig:div_rep}
\end{figure}

\subsubsection{Effect of number of samples to find the barycenter}
Similar to section \ref{ablation_aug}, we would like to study the effect of the number of data points used to find the barycenter. Keeping the augmentation factor fixed at 20, we vary the number of samples to find the barycenter. As can be seen in figure \ref{fig:ablation_bary}, it can be understood that as the data points to sample from increases the performance marginally drops. This becomes intuitive if we think of computing the barycenter as averaging over the samples and if we average out many samples we effectively get the representative sample which would similar in most iterations especially in the labels space. 

\subsubsection{Diversity and Representativeness}
We compute the diversity and representativeness of the selected samples as outlined in \cite{ein-dor-etal-2020-active}. From figure \ref{fig:div_rep} we see that our method gives comparable values of these metrics. DAL performs well on both the metrics as it was designed for maximising them. This shows there is some room for improvement in the proposed method with regards to the diversity and representativeness metrics. We leave this for future works.

\end{document}